\documentclass{article}
\usepackage{kr}

\usepackage{colortbl} 

\usepackage{times}
\usepackage{soul}
\usepackage{url}
\usepackage[utf8]{inputenc}
\usepackage[small]{caption}
\usepackage{graphicx}
\usepackage{amsmath}
\usepackage{amsthm}
\usepackage{booktabs}
\usepackage{algorithm}
\usepackage{algorithmic}
\usepackage{xcolor} 

\usepackage{xargs}

\usepackage{latexsym}

\usepackage{color}
\usepackage{amssymb}
\usepackage{soul}

	\newcommand\Cst{\mathcal{C}} 
	\newcommand\Prd{P} 
	\newcommand\Voc{(\Prd,\Cst)} 
	 	
	\newcommand\PS{S} 
	\newcommand\PO{O} 

	\newcommand\I{I} 
	\newcommand\D{D} 
	\newcommand\Mod{\mathcal I} 

	\newcommand{\terms}[1]{\mathtt{terms}(#1)} 
	\newcommand{\vars}[1]{\mathtt{vars}(#1)} 
	\newcommand{\consts}[1]{\mathtt{consts}(#1)} 

	\newcommand\q{Q} 
	\newcommand\ucq{Q} 
	\newcommand\cq{q} 
	\newcommand\qs{Q_S} 
	\newcommand\qo{Q_O} 

	\newcommand{\tuple}[1]{\vec{#1}}	

	\newcommand{\er}{R}	
	\newcommand{\erules}{\mathcal{R}} 

	\newcommand{\fr}[1]{\mathtt{fr}(#1)} 
	\newcommand{\body}[1]{\mathtt{body}(#1)} 
	\newcommand{\head}[1]{\mathtt{head}(#1)} 

	\newcommand\map{\mathcal{M}} 
	\newcommand\mapa{\mathcal{M}_\mathcal{A}} 
	\newcommand\rmap{\map_\mathcal{A}'} 

	\newcommand\KBDM{\Sigma} 
	\newcommand\kbdm{(\PS,\PO,\map,\erules)}
	
	\newcommand\mods{\mathtt{Mod}} 
	\newcommand{\modK}[2]{\mods_{#2}({#1})} 

	\newcommand{\homo}{h} 

	\newcommand{\eval}[2]{#1(#2)} 
	\newcommand{\cert}[2]{\mathtt{certain}(#1,#2)} 
        \newcommand{\certain}[3]{\mathtt{certain}_{#1}(#2, #3)}

	\newcommand{\evalC}[2]{#1^\mathtt{cert}(#2)}  
	\newcommand{\chase}[2]{\mathtt{chase}(#1,#2)} 
	\newcommand{\rew}[2]{\mathtt{rew}(#1,#2)}  
	\newcommand{\mrew}[1]{\map^-(#1)} 
 	
\newcommand{\vect}[1]{\tuple{#1}}

\newcommand{\sol}[2]{\modK{#2}{#1}}
\newcommand{\Srewriting}[2]{{#2}^-(#1)}

\newcommand{\target}{\mathcal{T}}
\newcommand{\source}{\mathcal{S}}

\usepackage{refcount}

\newcommand{\saveContent}[4][]{%
  \expandafter\def\csname content#3\endcsname{#4} 
  \begin{#2}[#1]\label{#3}
  #4
  \end{#2}
}
\newcommand{\recallContent}[2][Theorem]{
  \paragraph{#1~\ref{#2}.} \csname content#2\endcsname
  \par\addvspace{\baselineskip} }

\newtheorem{theorem}{Theorem}
\newtheorem{example}[theorem]{Example}
\newtheorem{definition}[theorem]{Definition}
\newtheorem{corollary}[theorem]{Corollary}
\newtheorem{proposition}[theorem]{Proposition}
\newtheorem{lemma}[theorem]{Lemma}

\title{Abstractions of Queries in Ontology-Based Data Access\\ (Extended version)
\footnote{\emph{This report contains the paper published at KR 2025 (with title ``Abstractions of Queries in Ontology-Based Data Access'' and the same list of authors) and an appendix with detailed proofs and further examples.}}
}

\author{%
Michel Leclère\and
Marie-Laure Mugnier\and
Guillaume Pérution-Kihli \\
\affiliations
LIRMM, Inria, University of Montpellier, CNRS, Montpellier, France 
\emails
\{michel.leclere, marie-laure.mugnier, guillaume.perution-kihli\}@lirmm.fr
}

\begin{document}

\maketitle

\begin{abstract} 
In ontology-based data access (OBDA), multiple data sources are integrated via mappings to an ontology. We consider an OBDA setting based on existential rules and the certain answer semantics. We address the recent issue of query abstraction, which consists of abstracting data queries by translating them to the ontology layer. Since a perfect abstraction may not exist, the notions of minimally complete and maximally sound abstractions have been introduced. 
 We study abstractions within an extension of UCQs with a limited form of inequality and a special predicate marking database constants. While this extension does not lead to an increased complexity of the problems of interest, it is able to express minimally complete abstractions, hence perfect abstractions when they exist. We also characterize maximally sound abstractions by making a new connection with the notion of  maximum recovery stemming from data exchange. 
\end{abstract}

\section{Introduction}

In ontology-based data access (OBDA), multiple data sources are integrated via mappings to a shared ontology \cite{Poggi2008LinkingDT}. This approach generalizes classical data integration by replacing the global schema with an ontology, enabling users to query data at a high level of abstraction while benefiting from reasoning over domain knowledge.

A central problem in OBDA is ontological query answering. Given an OBDA specification $\KBDM=(\PS,\PO,\map,\mathcal O)$, where $\PS$ is a data schema, $\PO$ an ontology schema, $\map$ a mapping from $\PS$ to $\PO$, and $\mathcal O$ an ontology over $\PO$, and a database $D$ over $\PS$, the goal is to compute the answers to an ontological query $Q_O$ over the knowledge base (KB) defined by $D$, $\map$, and $\mathcal O$. Such KB may remain virtual: then, $Q_O$ is rewritten into a data-level query $Q_S$ that is a \emph{perfect} translation of $Q_O$, which means that, for any database $D$ over $\PS$, the answers to $Q_S$ on $D$ coincide with the answers to $Q_O$ on the virtual KB. While $Q_S$ is simply evaluated on $D$, the answers to $Q_O$ are logically entailed by the KB—corresponding to the standard semantics of \emph{certain answers}, i.e., answers that hold in every model of the KB.

More recently, research in OBDA has also investigated query translation in the opposite direction, i.e., from data queries to ontological queries—a task called \emph{query abstraction} \cite{Cima2023TheNO}. This issue arises in a range of relevant scenarios. 
First, during the (often incremental) design of an OBDA system, query abstraction can be used to verify whether the mapping provides adequate coverage of important data queries \cite{Lutz_et_al_KR_2018}. 
Second, query abstraction is a means to automatically characterize the semantics
of data services implemented at the data-source level—which can be seen as a form of reverse engineering \cite{Cima2019SemanticCO}. This capability opens the door to promising applications, such as providing open datasets supplied by organizations with high-level semantics, or enhancing the FAIRness of data services \cite{Cima2023TheNO}. 

In this paper, we investigate query abstraction within an OBDA setting based on \emph{existential rules} \cite{blms:09,pods-09-cgl}, aka TGDs in database theory \cite{AbiteboulHV95}.
 We use existential rules for both the mapping (we obtain Global-Local-As-View (GLAV) mappings) and the ontology.  
  So doing, we generalize classical OBDA frameworks based on Horn description logics (DLs), as these can be seen as specific existential rule classes. Moreover, this uniform setting allows us to rely on the same fundamental tools to handle mappings and ontologies, namely the \emph{chase} and 
\emph{query rewriting}. It also helps to make connections with database theory.

Most work in OBDA has focused on ontological queries expressed as unions of conjunctive queries (UCQs), the core relational database queries. FO-rewritable ontologies—such as those expressed with the main dialects of the DL-Lite family \cite{DBLP:journals/jar/CalvaneseGLLR07} or certain fragments of existential rules \cite{blms:09,DBLP:journals/pvldb/CaliGP10}—guarantee that every ontological UCQ admits a perfect rewriting as a UCQ.
  In the other direction, query translation appears to be much more challenging. 
 To start with, a perfect abstraction of a data (U)CQ may not exist at all, even with an empty ontology. Apart from the fact that mappings may not transfer all the answers, they may also make source relations indistinguishable, as illustrated next.

\begin{example} Let $S = \{s_1,s_2\}$, $O = \{r\}$ and $\map = \{m_1,m_2\}$, with
$m_1$: $s_1(x) \rightarrow r(x)$ and $m_2$: $ s_2(x) \rightarrow r(x)$.\\
The query $Q_{\PS}(u) = s_1(u)$ has no perfect abstraction through $\map$. In particular, the ontological query $Q_{\PO}(u) = r(u)$ captures all answers to $Q_{\PS}$, i.e., it is a complete abstraction of $Q_{\PS}$, but it is not a sound abstraction of $Q_{\PS}$, as it also retrieves values coming from the source relation $s_2$. In fact, $Q_{\PO}$ would be a perfect abstraction of $s_1(u) \lor s_2(u)$. 
\label{ex-simple}
\end{example}

The topic began to be investigated only recently  \cite{Lutz_et_al_KR_2018,Cima2019SemanticCO,DBLP:conf/lics/CimaCLP21}. 
However, OBDA can build on a rich body of previous work in database theory.
As pointed out by Lutz et al. \cite{Lutz_et_al_KR_2018}, when the ontology is ignored, deciding if a source query has a perfect abstraction as a UCQ is closely related to the long-studied problem of query expressibility with views \cite{DBLP:journals/tods/NashSV10}.  
We will establish a new link with database theory by considering specific inverse mappings from data exchange 
\cite{fagin2008quasi,Arenas2009TheRO}. 

Since a perfect abstraction of a UCQ $Q_{\PS}$ 
may not exist, Cima et al.  introduced \emph{minimally complete} and \emph{maximally sound} abstractions, which respectively provide a minimal superset and a maximal subset of answers to $Q_{\PS}$ \cite{Cima2019SemanticCO}. Here, minimality and maximality are with respect to the set of queries definable in some target query language (e.g., UCQ). 

With the aim of capturing perfect abstractions of UCQs, we consider an extension of UCQ, denoted by UCQ$^{\neq,\textbf{C}}$, with a  limited form of inequality (variables that occur in inequalities must be mapped to constants) and a special unary predicate marking database constants, denoted by $\textbf{C}$---which was introduced in \cite{fagin2008quasi} to define inverse mappings. As a data query language, UCQ$^{\neq}$ is of great interest in practice, especially as inequalities are in fact not limited  when the queried databases are ground. However, the true benefit of  UCQ$^{\neq,\textbf{C}}$ is as an ontological query language. Indeed, while a perfect abstraction of a UCQ, when it exists, may not be expressible as a UCQ, we will show that it is always expressible as a UCQ$^{\neq,\textbf{C}}$.

We will now detail our main contributions.  To distinguish between the settings of database integration, i.e., without ontology, and OBDA, we will use the terms $\map$-abstraction and $\KBDM$-abstraction, respectively.

\smallskip
\noindent \textbf{(1) Complexity of the problems of interest within the UCQ$^{\neq,\textbf{C}}$ class. } 
In Section \ref{sec-ucq-dif}, we study the computational impact of extending UCQ to UCQ$^{\neq,\textbf{C}}$.
In particular, we show that this does not lead to increased complexity of 
verifying whether a candidate $\map$-abstraction is perfect, which is $\Pi^{\mathsf P}_2$-complete (see Table \ref{table-complexity}, first column). 
We also exhibit an FO-rewritable rule fragment for which this problem with a $\KBDM$-abstraction remains in $\Pi^{\mathsf P}_2$  (this subsumes earlier results on DL-Lite). Other complexity results are established later.  

 \begin{table} 
 \begin{small}
  \centering
 \begin{tabular}{|c|c|c|}
  \hline
  \textbf{Setting} &\textbf{Verification}&    \textbf{Existence}\\
  \hline
 {(U)CQ, GAV, } &  \textsc{$\Pi^{\mathsf P}_2$-c} (*) & \textsc{$\Pi^{\mathsf P}_2$-c} (*) \\
 {bounded arity} &  & \\
  \hline
{UCQ$^{\neq,\textbf{C}}$, GLAV}  &  \cellcolor[gray]{.8}  \textsc{$\Pi^{\mathsf P}_2$-c} &  \cellcolor[gray]{.8}  \textsc{$\Pi^{\mathsf P}_2$-c}\\
   {bounded frontier} &  \cellcolor[gray]{.8} &  \cellcolor[gray]{.8} \\
   \hline
  {UCQ$^{\neq,\textbf{C}}$, GLAV} &  \cellcolor[gray]{.8}  \textsc{$\Pi^{\mathsf P}_2$-c} &  \cellcolor[gray]{.8}  in  \\
  &  \cellcolor[gray]{.8}&  \cellcolor[gray]{.8} \textsc{Co-NExpTime} \\
  \hline
 \end{tabular}\\
 \small{(*) from \cite{Lutz_et_al_KR_2018}}
 \caption{Complexity of problems related to perfect $\map$-abstractions } \label{table-complexity}
 \end{small}

\end{table}

\smallskip
\noindent \textbf{(2) Capturing minimally-complete and perfect abstractions of UCQ$^{\neq,\textbf{C}}$s.} In Section \ref{sec-min-complete}, we first point out that a non-Boolean query $Q_{\PS}$ may not have a \emph{complete} abstraction and characterize the conditions under which it has one (which only depends on the interactions between $Q_{\PS}$ and the mapping).  For a UCQ $Q_{\PS}$ that has a complete abstraction, it is known that, by applying $\map$ to $Q_{\PS}$ (i.e., ``chasing''  $Q_{\PS}$ with $\map$), one produces a UCQ $Q_{\PO}$ that is \emph{minimally complete} w.r.t. all $\map$-abstractions definable as UCQs \cite{Lutz_et_al_KR_2018,Cima2019SemanticCO}. However, we point out that $Q_{\PS}$ may have a better $\map$-abstraction (i.e., a complete $\map$-abstraction with fewer unwanted answers) in the UCQ$^{\neq,\textbf{C}}$ class, which can be computed by a modified chase. The important result here is that  UCQ$^{\neq,\textbf{C}}$ is in fact able to express a minimally complete $\map$-abstraction (and $\KBDM$-abstraction as well) of any source UCQ$^{\neq,\textbf{C}}$, where minimality is w.r.t. all possible ontological queries (i.e., queries on schema $O$ expressed in any language, provided that they are answered
with the semantics of certain answers). We take special care in exhibiting the properties behind this result.  
It follows that, when a perfect $\KBDM$-abstraction of a UCQ$^{\neq,\textbf{C}}$ exists, it can be expressed in this class. 
Moreover, UCQ$^{\neq,\textbf{C}}$ is a minimal language with this property, even when the source query is a plain UCQ. 
Finally, we use these results to show that the complexity of deciding whether a perfect $\map$-abstraction exists for a given UCQ$^{\neq,\textbf{C}}$ is $\Pi^{\mathsf P}_2$-complete when the mapping rules have a frontier of bounded size (the frontier being the set of variables shared between the body and head) and  in Co-NExpTime otherwise (see Table \ref{table-complexity}, second column). 

\smallskip
\noindent \textbf{(3) Characterizing maximally sound abstractions of UCQ$^{\neq,\textbf{C}}$s.} There is no known algorithm that builds a \emph{maximally sound} $\map$-abstraction of a UCQ as a UCQ (or UCQ$^{\neq,\textbf{C}}$), when such abstraction exists, except in a very specific case \cite{Cima2019SemanticCO}. Whether the associated existence problem is decidable is an open question. In Section \ref{sec-max-sound}, we make a step towards better understanding by drawing a connection with the notion of a \emph{maximum recovery} investigated in a quite different context, namely data exchange. In data exchange, mappings are used to specify how to transfer data from a source schema to a target schema; then, a maximum recovery of a mapping $\map$ is an inverse mapping from target to source that allows one to recover the most answers to source queries when it is composed with $\map$  \cite{Arenas2009TheRO}. To express such inverse mapping, a strictly more expressive language than GLAV is required, as disjunction in rule heads is needed. 
We show that a maximally sound $\map$-abstraction of a UCQ$^{\neq,\textbf{C}}$ $Q_{\PS}$ can be equivalently defined as the \emph{rewriting} of $Q_{\PS}$ with a maximum recovery of $\map$. To define such rewriting, we rely on a rewriting operator for UCQs and disjunctive existential rules introduced in  \cite{Leclre2023QueryRW}. 
In passing, we correct a wrong claim from the literature, that a maximum recovery for CQs could always be expressed by a conjunctive mapping (i.e., without disjunction) \cite{Arenas2009InvertingSM}. We then extend the notion of a maximum recovery to an OBDA specification (i.e., we add an ontology) and show that a maximum recovery can still be expressed by the same form of disjunctive mapping when the ontology is FO-rewritable; hence, in this case, a maximally sound $\KBDM$-abstraction of a UCQ$^{\neq,\textbf{C}}$ can also be characterized as a rewriting with a maximum recovery. When this rewriting is finite, it is a UCQ$^{\neq,\textbf{C}}$. 


\section{Preliminaries}

We assume the reader has basic knowledge in database theory and ontological query answering. 

\smallskip
\noindent\textbf{Database theory}
We consider a denumerable set of constants $\Cst$. A \emph{term} is a constant from $\Cst$ or a variable. A \emph{schema} is a finite set of predicates. 
 An atom with predicate $p$ is called a \emph{$p$-atom}. Given a schema $\Prd$, a \emph{$\Prd$-atom} is any $p$-atom with $p \in \Prd$. 
A \emph{$\Prd$-instance} is a possibly infinite set of $\Prd$-atoms, also seen as an \emph{interpretation} of $\Voc$, in which constants are interpreted by themselves. 
A \emph{database} is a ground finite instance. 
Given an instance $\I$, we denote by $\vars{\I}$, $\consts{\I}$ and $\terms{\I}$ its sets of variables, constants and terms, respectively. 
A tuple $\tuple{x}$ of pairwise distinct variables is sometimes seen as a set.  
 Given instances $\I_1$ and $\I_2$,
a \emph{homomorphism} $\homo$ from $\I_1$ to $\I_2$ is a substitution of $\vars{\I_1}$
by $\terms{\I_2}$ such that $\homo(\I_1) \subseteq \I_2$.  We also denote it by $\homo: \I_1 \rightarrow \I_2$ and we say that $\I_1$ \emph{maps} to  $\I_2$ (by $\homo$). When convenient, we extend the domain of a homomorphism to constants.

We define the notion of a query in an abstract way, i.e., independently from any syntactic form: an $n$-ary \emph{query} $\q$ 
on schema $\Prd$
is
any function that maps each 
$\Prd$-instance $\I$ to a 
set of $n$-ary tuples on $\consts{\I}$.
We denote by $\eval{\q}{\I}$ the set of answers to $\q$ on $\I$.
An  $n$-ary \emph{FO-query} on $\Prd$ has the form $\q(\tuple x)$, where $\q$ is an FO-formula on $\Prd$ and $\tuple x$ are $n$ free variables from $\q$ (the answer variables). A \emph{Boolean} query is 0-ary. 
Given an instance $\I$, a tuple $\tuple{c} \subseteq \consts{\I}$ with $|\tuple{c}| = |\tuple{x}|$ is an \emph{answer} to an FO-query $\q(\tuple{x})$ on $\I$ if $\I$ is a model of $\q(\tuple{c})$, where  $\q(\tuple{c})$ is obtained by substituting each i-th variable in $\tuple{x}$ by the i-th constant in $\tuple{c}$. %
A \emph{conjunctive query} (CQ) has the form $\cq(\tuple{x}) = \exists \tuple{y}.~\phi[\tuple{x},\tuple{y}]$, where 
 $\phi$ is a finite conjunction of atoms
and $\tuple{x} \cup \tuple{y} = \vars{\phi}$;  to denote a CQ, we write $q$ rather than $Q$. 
A CQ is \emph{with join-free existential variables} (CQJFE) if each existential variable occurs only once. 
A \emph{union of conjunctive queries} (UCQ) is a finite disjunction of CQs with the same tuple of answer variables.
 A UCQ may contain equalities, but to simplify technical developments we will silently assume that, before any processing, equalities are removed by substituting variables. 

Given queries $\q_1$ and $\q_2$, 
$\q_1$ is \emph{contained} in $\q_2$, denoted by $\q_1 \sqsubseteq \q_2$, if $\eval{\q_1}{\I}\subseteq\eval{\q_2}{\I}$ for any instance $\I$. 
Given two CQs $\cq_1(\tuple x_1)$ and $\cq_2(\tuple x_2)$, a \emph{query homomorphism} from $\cq_1$ to $\cq_2$ is a homomorphism $h$ from $\cq_1$ to $\cq_2$ such that $h(\tuple x_1) = \tuple x_2$.
It is well known that, given CQs $\cq_1$ and $\cq_2$, $\cq_1 \sqsubseteq \cq_2$ iff there is a query homomorphism from $\cq_2$ to $\cq_1$.  Note that for UCQs $\ucq_1$ and $\ucq_2$, we have $\ucq_1 \sqsubseteq \ucq_2$ iff for all $\cq_1\in \ucq_1$, there is $\cq_2\in \ucq_2$ such that $\cq_1 \sqsubseteq \cq_2$.

When information is incomplete, a set of instances is often considered instead of a single instance; then, the set of \emph{certain answers} to a query $\q$ on a set of instances $\Mod$ is $\cert{\q}{\Mod} = \bigcap\limits_{\I\in \Mod} \eval{\q}{\I}$.

\noindent\textbf{Rules and mappings}
An \emph{existential rule} $\er$ (or simply \emph{rule} hereafter) 
is a closed formula of the form  
  $\forall\tuple{x}.~ (\exists \tuple{y}.~ B[\tuple{x},\tuple{y}]) \rightarrow \exists \tuple{z}.~ H[\tuple{x},\tuple{z}]$
 where $B$ and $H \neq \emptyset$ are finite conjunctions, respectively called the \emph{body} and the \emph{head} of $R$, also noted $\body{\er}$ and $\head{\er}$, and $\tuple{x}, \tuple{y}$ and $\tuple{z}$ are pairwise disjoint tuples of variables and $\consts{H}\subseteq \consts{B}$. 
The \emph{frontier} of $\er$ is $\fr{\er} = \tuple{x}$. Note that $\body{\er}$ and $\head{\er}$ can be seen as CQs with answer variables $\fr{\er}$.
 $\er$ is \emph{Datalog} if  $\tuple{z} = \emptyset$. 
Given schemas $\mathcal S$ and $\mathcal T$, called source and target respectively, with $\mathcal S \cap \mathcal T = \emptyset$, an \emph{$\mathcal S$-to-$\mathcal T$ rule} has a body made of $\mathcal S$-atoms and a head made of $\mathcal T$-atoms.  
A GLAV \emph{mapping} from $\mathcal S$ to $\mathcal T$ is a finite set of $\mathcal S$-to-$\mathcal T$ rules. 
A GLAV mapping is \emph{GAV} if it is a set of Datalog rules.

\smallskip
\noindent\textbf{OBDA} An \emph{OBDA specification} is a quadruplet $\KBDM=\kbdm$ where $\PS$ is the source schema, $\PO$ the ontology schema, $\map$ a (GLAV) mapping from $\PS$ to $\PO$ and $\erules$ a finite set of  rules over $\PO$.
An \emph{OBDA system} is a pair $(\D,\KBDM)$ with $\KBDM$ an OBDA specification and $\D$ an $\PS$-database. 
A \emph{source query}, usually denoted by $\qs$, is defined on $S$ and an \emph{ontological query}, usually denoted by $\qo$, is defined on $O$. 
The answers to an ontological query $\qo$ on an OBDA system $(\D,\KBDM)$, denoted by $\evalC{\qo}{\D,\KBDM}$, 
 are its certain answers
 on the \emph{models of the OBDA system}, 
denoted by $\modK{\D}{\KBDM}$:
$$\modK{\D}{\KBDM} ~=~ \{\PO\textit{-instance}\, \I~|~ \D\cup\I\models\map \,\textit{and}\, \I\models\erules\}$$
$$\evalC{\qo}{\D,\KBDM} = \cert{\qo}{\modK{\D}{\KBDM}}$$
When we want to ignore the set of rules $\erules$ of an OBDA system, we use notations with $\map$ instead of $\Sigma$.
Query containment is extended to ontological queries: $\q_1\sqsubseteq_\KBDM\q_2$ if $\evalC{\q_1}{\D,\KBDM}\subseteq \evalC{\q_2}{\D,\KBDM}$ for every $\PS$-database $\D$.

\smallskip
\noindent\textbf{Query abstraction}
For an OBDA specification $\KBDM$, a query $\qo$ on $O$ is a \emph{complete} (resp. \emph{sound}) $\KBDM$-abstraction of a query $\qs$ on $S$ if $\eval{\qs}{\D}\subseteq\evalC{\qo}{\D,\KBDM}$ (resp. $\evalC{\qo}{\D,\KBDM}\subseteq \eval{\qs}{\D}$) for all $\D$; $\qo$ is a \emph{perfect} $\KBDM$-abstraction if it is both sound and complete. 
We also consider ``best'' approximations of perfect abstractions within a target query class: Given a query class $\mathbb Q$, an ontological query $\qo \in \mathbb Q$ is a \emph{$\mathbb Q$-minimally complete} $\KBDM$-abstraction of $\qs$ if $\qo$ is a complete $\KBDM$-abstraction of $\qs$ and $\qo\sqsubseteq_\KBDM\qo'$ for any $\qo' \in \mathbb Q$ that is a complete $\KBDM$-abstraction of $\qs$; similarly, $\qo$ is a 
\emph{$\mathbb Q$-maximally sound} $\KBDM$-abstraction of $\qs$ if $\qo$ is a sound $\KBDM$-abstraction of $\qs$ and $\qo'\sqsubseteq_\KBDM\qo$ for any $\qo' \in \mathbb Q$ that is a sound $\KBDM$-abstraction of $\qs$. 
\emph{When we do not specify $\mathbb Q$, we consider \emph{any} query, as abstractly defined above.  }
Finally, when $\erules$ is ignored, we speak of \emph{$\map$-abstraction} instead of $\KBDM$-abstraction. 

\smallskip
Let $\mathcal X \in \{\map, \Sigma \}$ and $\mathbb{Q}$ be a class of queries. We study three kinds of problems: 
\begin{itemize}
\item \emph{Verifying} if an abstraction is perfect: the $\mathbb{Q}$ $\mathcal X$-perfectness \emph{verification} problem takes as input $\mathcal X$, $\mathbb{Q}$-queries $\qs$ and $\qo$, and asks if $\qo$ is a perfect $\mathcal X$-abstraction of $\qs$. This problem is decomposed into verifying $\mathcal X$-soundness and 
 $\mathcal X$-completeness.
\item \emph{Deciding} the existence of a perfect abstraction: the $\mathbb{Q}$ $\mathcal X$-\emph{expressibility} problem takes as input  $\mathcal X$ and $\qs \in \mathbb{Q}$, and asks if a perfect $\mathcal X$-abstraction of $\qs$ is expressible in  $\mathbb{Q}$. 
\item \emph{Computing} abstractions satisfying property $P \in \{ \text{perfectness, maximal soundness, minimal  completeness} \}$; Given $\mathcal X$ and a $\mathbb{Q}$-query $\qs$, the task is to compute a $\mathbb{Q}$-query $\qo$ that is an $\mathcal X$-abstraction of $\qs$ with property $P$, 
when such an abstraction  exists. 
\end{itemize}

\smallskip
\noindent\textbf{Reasoning tools}
Let $\erules$ be a set of rules. 
 Given an instance $I$, the \emph{chase} with $\erules$ exhaustively applies rules from $\erules$ to $I$, towards a fixpoint.
 \footnote{This formulation corresponds to the simplest chase variant, called oblivious \cite{DBLP:conf/kr/CaliGK08}.}
 We denote by $\chase{I}{\erules}$ the (possibly infinite) resulting instance. 
Crucially, $\chase{I}{\erules}$  is a \emph{universal} model of $I$ and $\erules$, i.e., it maps to any model of $I$ and $\erules$. 
Given a UCQ $\ucq$, \emph{query rewriting} with $\erules$ starts from $Q$ seen as a set of CQs; it iteratively rewrites a CQ from the set with a rule from $\erules$ and adds the resulting CQ to the set, while keeping a minimal set w.r.t. query containment, towards a fixpoint. 
We consider here the rewriting algorithm from \cite{SWJ15}, based on so-called piece-unifiers,  and denote by $\rew{\ucq}{\erules}$ the possibly infinite set (i.e., union) of CQs it produces. Each rewriting step is based on a piece-unifier, which unifies a subset $q'$ of a CQ $q$ with a subset of a rule head $h'$ while satisfying the following \emph{piece condition}: an existential variable from $h'$ can only be unified with variables from $q'$, which furthermore do not occur in $q \setminus q'$. 
 A fundamental property holds: For any instance $I$, set of rules $\erules$ and  Boolean UCQ $Q$, $I \cup \erules$ entails $Q$ iff  $\chase{I}{\erules} \models Q$ iff $I \models  \rew{\ucq}{\erules}$.
 A pair $(\ucq, \erules)$ where $\ucq$ is a UCQ, is \emph{FO-rewritable} if there is a UCQ $\ucq'$ such that $\ucq'(I)=
\cert{\ucq}{\chase{I}{\erules}}$ for any instance $I$. 
Note that $(\ucq, \erules)$ is FO-rewritable iff $\rew{\ucq}{\erules}$ is finite. A rule set $\erules$ is FO-rewritable, or \emph{fus}, if $(\ucq, \erules)$ is FO-rewritable for any UCQ $\ucq$.

When mappings (from $\mathcal S$ to $\mathcal T$) are considered instead of rules on a single schema, the result of the chase or of query rewriting is restricted to the relevant schema ($\mathcal T$ or $\mathcal S$). 
Given a mapping $\map$ from $\source$ to $\target$ and a (finite) $\source$-instance $I$, the chase of $I$ with $\map$ yields the  $\target$-instance 
$\map(I)=\{A\in\chase{I}{\map} ~|~ A \text{~is~a~} \target\text{-atom}\}$. 
Given a CQ $\cq_s$, the chase of $\cq_s$ with $\map$ is defined similarly, on the atoms of $\cq_s$ seen as an instance, 
provided that each answer variable from $\cq_s$ occurs in a $\target\text{-atom}$ of $\map(\cq_s)$; hence, $\map(\cq_s)$ can be seen as a CQ with the same arity as $\cq_s$. The chase is further extended to a UCQ: given an $\source$-UCQ $\ucq_s$, 
$\map(\ucq_s)$ is defined only if $\map(\cq_{s_i})$ is defined for each $\cq_{s_i}\in \ucq_s$; then $\map(\ucq_s)$ 
is the $\target$-UCQ obtained by making the disjunction of the $\map(\cq_{s_i})$, for all $\cq_{s_i}\in \ucq_s$. 
Query rewriting with a mapping $\map$ takes as input a $\target$-UCQ $\ucq_t$ and produces the $\source$-UCQ 
$\map^-(\ucq_t)=\{\cq\in\rew{\ucq_t}{\map} ~|~ \cq \text{~is~a~query~on~} \source\}$. 
Note that $\map(\ucq_s)$ and $\map^-(\ucq_t)$ are always finite. 
  
It is convenient to use the same notations for mappings and general rule sets; so we also note  $\erules(I)=\chase{I}{\erules}$ and $\erules^-(\ucq)=\rew{\ucq}{\erules}$. 
For an OBDA specification $\KBDM = \kbdm$, we note $\KBDM(I)=\erules(\map(I))$ and $\KBDM^-(\ucq)=\map^-(\erules^-(\ucq))$.  When $\erules^-(\ucq)$ is not finite, it is seen as an infinitary query, as in \cite{Lutz_et_al_KR_2018}.

\smallskip
\noindent\textbf{Properties of OBDA systems}
We finally list some fundamental properties of OBDA systems, which are explicit or implicit in previous work, except that we consider existential rules instead of specific Horn description logics. 

\begin{proposition}\label{prop-known}
For any OBDA specification $\KBDM = \kbdm$,
$\PS$-database $D$, $\PS$-UCQ $\qs$ and $\PO$-UCQ $\qo$ the following holds:
\begin{enumerate}
\item $\evalC{\qo}{\D,\map}\subseteq \evalC{\qo}{\D,\KBDM}$
\item $\evalC{\qo}{D,\Sigma}~=~  ~\eval{\qo}{\erules(\map(D))}~=~ \eval{(\map^-(\erules^-(\qo)))}{D}$ 
\item If $\qs\sqsubseteq \qs'$ then $\map(\qs) \sqsubseteq_\map \map(\qs')$
\item $\qo \sqsubseteq_\map \qo'$ iff $\map^-(\qo) \sqsubseteq \map^-(\qo')$  
\item $\qo \sqsubseteq_\Sigma \qo'$ iff $\erules^-(\qo) \sqsubseteq_\map \erules^-(\qo')$  
\end{enumerate}
\end{proposition}

When a perfect $\map$-abstraction of a UCQ $\qs$ is expressible as a UCQ, $\map(\qs)$ is such an abstraction (this follows, e.g., from \cite{DBLP:journals/tods/NashSV10}). 
Moreover, when $\qs$ has a complete $\map$-abstraction, $\map(\qs)$ is such an abstraction, which is even UCQ-minimally complete \cite{Cima2019SemanticCO}.  
Hence, a UCQ $\qs$ has a perfect  $\map$-abstraction expressible as a UCQ iff $\map(\qs)$ is sound, which can be checked by verifying if $\map^-(\map(\qs)) \sqsubseteq \qs$. This result has been extended to a $\Sigma$-abstraction in some specific OBDA settings.\footnote{GAV mappings \cite{Lutz_et_al_KR_2018} or ontology in DL-Lite$_\mathcal R$ \cite{Cima2019SemanticCO}.}

\smallskip
Next, we use the following notations: $\KBDM$ is an OBDA specification defined by $\kbdm$, $D$ is an $\PS$-database, and $\qs$, $\qo$ are queries over $\PS$ and $\PO$ respectively.

\section{From UCQ to UCQ$^{\neq,\textbf{C}}$}
\label{sec-ucq-dif}

The class UCQ$^{\neq,\textbf{C}}$ extends UCQ with two special predicates: a restricted form of inequality ($\neq$) and a unary predicate $\textbf{C}$ stating that its argument is a constant.\footnote{ Formally: for any FO-interpretation $\mathcal I = (\Delta, .^{\mathcal I})$ 
it holds that $\textbf{C}^{\mathcal I} = \consts{\mathcal I}$ and $\neq^\mathcal{I} $ = $ \{(d_1,d_2) \in \Delta^2 ~|~d_1 \neq d_2 \}$.}
In the context of query abstraction, $\textbf{C}$ is used to mark variables in ontological queries that must be mapped to values coming from the database, whereas $\neq$ allows one to distinguish between different ways of matching query variables. 

\begin{definition}[UCQ$^{\neq,\textbf{C}}$] \label{def-ucq-dif} A UCQ$^{\neq,\textbf{C}}$  $Q$ is a UCQ extended with atoms on special predicates $\neq$ (binary) and $\textbf{C}$ (unary), such that: (1) All the variables of $Q$ occur in standard atoms, and (2) The terms of any $\neq$-atom are constants, answer variables, or variables that occur in a $\textbf{C}$-atom. 
We denote by $std(Q)$ the restriction of $Q$ to standard atoms. 
\end{definition}

 The $\textbf{C}$-atoms on answer variables and constants, as well as $\neq$-atoms over constants, can be made explicit or not. Also, $\textbf{C}$ is useless in source queries since databases are ground, but for simplicity we keep the same query class at  the data and the ontology levels. 
A UCQ$^{\neq,\textbf{C}}$ $Q$ is \emph{consistent} if there exists a database $D$ such that $Q(D) \neq \emptyset$, which can be checked in polynomial time.    
Next, we implicitly assume that queries are consistent.

All the technical tools for CQs are extended to CQ$^{\neq,\textbf{C}}$ in the natural way. A \emph{homomorphism} $h: \cq \rightarrow I$ is a homomorphism from $std(\cq)$ to $I$ such that (i) for all $\textbf C(t) \in \cq$, $h(t)$ is a constant, and (ii) for all $t_1 \neq t_2$, $h(t_1) \neq h(t_2)$.  A \emph{query homomorphism} $h: \cq_1 \rightarrow \cq_2$ is a homomorphism from $std(\cq_1)$ to $std(\cq_2)$ such that  (i) for all $\textbf C(t) \in \cq_1$, $h(\textbf C(t)) \in \cq_2$ or $h(t)$ is a constant or $h(t)$ is an answer variable, and (ii) for all $t_1 \neq t_2  \in \cq_1$, $h(t_1 \neq t_2) \in Q_2$ or $h(t_1)$ and $h(t_2)$ are distinct constants. 
The chase of $\cq$ with $\map$, i.e., $\map(\cq)$, is obtained from 
$\mathcal M(std(q))$ 
by (1) adding the  atoms $(t_1 \neq t_2)$ from $q$ if $t_1$ and $t_2$ are both in $\map(std(\cq))$, and (2) adding $\textbf C(x)$ on the variables from $\vars{q} ~\cap~\vars{\map(std(\cq))}$. Hence, if the $\textbf C$-atoms in $\cq$ are made explicit, $\map(\cq)$ is defined as for a plain (U)CQ w.r.t. target predicates  $\mathcal T \cup \{ \textbf{C}, \neq \}$. 
Similarly, the rewriting of $\cq$ with $\map$,  i.e., $\map^-(\cq)$, is defined w.r.t. source predicates $\mathcal S \cup \{ \textbf{C}, \neq \}$; 
inconsistent CQs can be removed from $\map^-(\cq)$ and  the $\textbf{C}$-atoms can be made implicit. 
Note that, by definition of a piece-unifier, an existential variable from a rule head cannot be unified with a variable that occurs in a  $\textbf{C}$- or  $\neq$-atom. These notions are illustrated in Ex. \ref{ex-ucqdifc}. They are extended to UCQ$^{\neq,\textbf{C}}$ as expected.  
Note that the fundamental properties from Prop. \ref{prop-known} still hold. 

\begin{example} Let $\mathcal M$ be the following mapping:\\
\begin{tabular}{ll}
$m_1: s_1(x,y) \rightarrow p(x,y)$ & $m_3: s_2(x) \rightarrow p(x,x)$\\
$m_2: s_1(x,x) \rightarrow r(x)$ & $m_4: s_3(x) \rightarrow \exists y. p(x,y)$
\end{tabular} 

Let the source CQ$^{\neq}$ $q(u) = \exists v. ~s_1(u,v)~\land~u \neq v$. Then, $\map(q(u)) = p(u,v) \land ~u \neq v ~\land ~\textbf{C}(v)$ (with $\textbf{C}(u)$ left implicit as $u$ is an answer variable) and $\map^-(\map(q(u)) \equiv  q(u)$. Let us detail the rewriting of $\map(q(u))$: $p(u,v)$ can be unified with $\head{m_1}$, which yields $q(u)$, as well as with $\head{m_3}$, which yields the inconsistent CQ $s_2(u) \land u \neq u$ (discarded);  $p(u,v)$ cannot be unified with $\head{m_4}$ because $v$ would be unified with the existential variable $y$, while it also occurs in $\textbf{C}(v)$ and in $u \neq v$.  
\label{ex-ucqdifc}
\end{example}

\vspace*{-0.1cm}

The next example shows that a UCQ may have no perfect abstraction in UCQ but one in UCQ$^{\neq,\textbf{C}}$.

\begin{example}[Perfect abstraction within UCQ$^{\neq,\textbf{C}}$] \label{ex-perfect-ucq-dif}
Consider again $\mathcal M$ from Ex. \ref{ex-ucqdifc}. 
The CQ $\cq_S(u) = \exists v. s_1(u,v)$ has no perfect $\mathcal M$-abstraction in UCQ. Indeed, $\mathcal M(\cq_S) = \exists v. p(u,v)$, and $\qs' = \mathcal M^-(\mathcal M(\cq_S)) = \exists v. s_1(u,v) ~\lor ~s_2(u) \lor~s_3(u)$ strictly contains $\cq_S$. 
Hence, $\mathcal M(\cq_S)$ is not a sound abstraction. 
However, the following UCQ$^{\neq,\textbf{C}}$ is a perfect abstraction of $\cq_S$:
$\qo(u)  = q_O^1(u) \lor q_O^2(u)$, with $q_O^1(u) = \exists v. ~p(u,v)  \land \textbf{C}(v) \land u \neq v$ and  $q_O^2(u) = r(u) \land p(u,u)$.
Such query $\qo$ is the output of the algorithm given in Section \ref{sec-min-complete}. 
Intuitively, $\qo$ is a sound abstraction because it does not retrieve the $p$-atoms produced by $m_3$ and $m_4$. Indeed, $q_O^1$ only retrieves $p$-atoms produced by $m_1$, thanks to  $u \neq v$ and $\textbf{C}(v)$. However, $q_O^1$ is not a complete abstraction as it avoids atoms of the form $p(a,a)$ that can be produced by $m_1$.  
Subquery $q_O^2$ compensates for this elimination. Note that the atom $p(u,u)$ in $q_O^2$ is actually not needed: indeed, when an atom $r(a)$ is produced (by $m_2$), the atom $p(a,a)$ is necessarily produced (by $m_1$). 
More formally, let us check that $\mathcal M^-(\qo) \equiv \cq_S$. The rewriting of $q_O^1$ yields the CQ $\exists v. ~s_1(u,v) \land u \neq v$, see Ex. \ref{ex-ucqdifc}. The rewriting of $q_O^2$ yields two CQs: 1. $s_1(u,u)$, obtained by unifying $r(u)$ with $\head{m_2}$ and $p(u,u)$ with $\head{m_1}$; and 2. $(s_1(u,u) \land s_2(u))$, obtained by unifying $r(u)$  with $\head{m_2}$ and $p(u,u)$ with $\head{m_3}$. The latter query is contained in the former, hence can be ignored. We obtain 
$\mathcal M^-(\qo) = (\exists v. ~s_1(u,v) \land u \neq v) \lor  s_1(u,u) \equiv \cq_S$. 
 \end{example}

A natural question is whether this extension increases the complexities of the problems we are interested in. 
 First note that a homomorphism from a CQ$^{\neq,\textbf{C}}$ $\cq$ to an instance $I$ necessarily maps terms from a $\neq$-atom to terms known to be constants. It follows that, for $\cq$ Boolean, if $I$ entails $\cq$ then $\cq$ maps to $I$ 
 (and reciprocally); hence, query answering can still rely on homomorphism. 
This is different for query containment: given CQs$^{\neq,\textbf{C}}$ $\cq_1$ and $\cq_2$, the existence of a query homomorphism from $\cq_2$ to $\cq_1$ is no longer a necessary condition for $\cq_1 \sqsubseteq \cq_2$, even when considering only databases.
   Let \emph{UCQ$^{\neq,\textbf{C}}$ containment} be the problem that takes as input two queries $Q_1$ and $Q_2$ in UCQ$^{\neq,\textbf{C}}$, and asks if $Q_1 \sqsubseteq Q_2$. 
  This problem is known to be  $\Pi^{\mathsf P}_2$-complete, already when both queries are in CQ$^{\neq,\textbf{C}}$  
\cite{DBLP:journals/jcss/Meyden97,DBLP:conf/pods/KolaitisMT98}. 
As a complementary result, we prove that  $\Pi^{\mathsf P}_2$-hardness already holds when $Q_1$ is a very simple kind of CQ.

\begin{theorem}[Complexity of UCQ$^{\neq,\textbf{C}}$ containment] \label{th-QC}
 The UCQ$^{\neq,\textbf{C}}$ containment 
 problem is $\Pi^{\mathsf P}_2$-hard when  $Q_1$ is a Boolean CQJFE and $Q_2$ is a Boolean CQ$^{\neq,\textbf{C}}$.   
\end{theorem}

\vspace*{-0.3cm}
\begin{proof}[Proof sketch] By a reduction from $\forall\exists$3CNF adapted from \cite{DBLP:journals/tcs/AbiteboulKG91}. 
\end{proof}

\vspace*{-0.2cm}
We now study the complexity of $\mathcal M$-perfectness verification, by decomposing that problem into 
$\mathcal M$-completeness and $\mathcal M$-soundness verifications. 
 The $\mathcal M$-completeness (resp. $\mathcal M$-soundness) verification problem can be recast as verifying if $Q_S \sqsubseteq \mathcal M^-(Q_O)$ (resp. $\mathcal M^-(Q_O) \sqsubseteq Q_S$). There is an immediate reduction from UCQ$^{\neq,\textbf{C}}$ containment to verification, taking a trivial mapping $\mathcal M$  that bijectively translates n-ary predicates in $S$ into n-ary predicates in $O$.

\begin{theorem}[Complexity of $\mathcal M$-completeness] \label{complexity-M-completeness}The UCQ$^{\neq,\textbf{C}}$ $\mathcal M$-completeness verification problem is $\Pi^{\mathsf P}_2$-complete, even if  $Q_S$ is a CQJFE 
and $Q_O$ is a CQ$^{\neq}$. 
\end{theorem}

\vspace*{-0.3cm}
\begin{proof} To verify that $Q_S \sqsubseteq \mathcal M^-(Q_O)$, we can check if, for every ground instantiation $D$ of a CQ$^{\neq}$ from $Q_S$, 
there is a CQ$^{\neq}$ $\cq_i \in \mathcal M^-(Q_O)$ that maps to $D$ (with answer variables mapped correctly). We can universally choose a $D$ in polynomial time as it is given by a substitution of the variables of a $q_s \in Q_S$ by fresh constants 
and we can guess $\cq_i$ and a homomorphism from $\cq_i$ to $D$ in polynomial time. Indeed, to obtain a $\cq_i$, we guess a CQ$^{\neq}$ $q_j \in Q_O$, a subset of $\mathcal M$ with at most  $|q_j|$ rules and associated piece-unifiers.  
Hence, $\mathcal M$-completeness is in $\Pi^{\mathsf P}_2$. Hardness follows from Th. \ref{th-QC}. 
\end{proof}

\vspace*{-0.2cm}
\begin{theorem}[Complexity of $\mathcal M$-soundness] \label{complexity-M-soundness}The UCQ$^{\neq,\textbf{C}}$ $\mathcal M$-soundness verification problem is $\Pi^{\mathsf P}_2$-complete, even if $Q_S$ is a Boolean CQ$^{\neq}$ and $Q_O$ is a is a Boolean CQJFE.  
\end{theorem}

\vspace*{-0.2cm}
\begin{proof} Similar to that of Th. \ref{complexity-M-completeness}.
\end{proof}

\vspace*{-0.2cm}
Since $Q_1\!\sqsubseteq\!Q_2$ iff $Q_1 \equiv Q_1\!\land\!Q_2$, query equivalence is as hard as query containment, hence: 

\begin{corollary} The UCQ$^{\neq,\textbf{C}}$ $\mathcal M$-perfectness verification problem is $\Pi^{\mathsf P}_2$-complete. 
\end{corollary}

It is known that $\mathcal M$-perfectness verification is $\Pi^{\mathsf P}_2$-hard already for $Q_S$ and $Q_O$ CQs and $\mathcal M$ a GAV mapping in a DL setting, i.e., with mapping heads restricted to unary and binary predicates  \cite{Lutz_et_al_KR_2018,Cima2019SemanticCO}. Hence, considering (U)CQ$^{\neq,\textbf{C}}$ (and GLAV mappings) does not lead to increased complexity of verification. 

\smallskip
When it comes to taking an ontology into account, most previous works have considered lightweight DLs that are \emph{fus}.\footnote
{An exception is \cite{Lutz_et_al_KR_2018} considering also non-fus DLs from the $\mathcal {EL}$ family.} 
The key point is that, for any \emph{fus} rule set  $\mathcal R$ and $\qo$ in UCQ$^{\neq,\textbf{C}}$, $\mathcal R^-(\qo)$ is also in UCQ$^{\neq,\textbf{C}}$; hence the techniques designed for $\map$-abstractions can be extended, however at the risk of increased complexity. 
Next, we show that perfectness verification  remains in $\Pi^{\mathsf P}_2$ when $\mathcal R$ is a set of \emph{linear} rules---i.e., existential rules whose body has a single atom---over predicates with bounded arity. 
This rule class generalizes several dialects of the DL-Lite family, in particular DL-Lite$_{\mathcal R}$ \cite{pods-09-cgl}. 

\begin{theorem} The UCQ$^{\neq,\textbf{C}}$ $\Sigma$-perfectness verification problem 
is in $\Pi^{\mathsf P}_2$ when $\erules$ is linear with bounded-arity predicates. 
\label{th-complexity-sigma-perfectness-verification}
\end{theorem}

\begin{proof}[Proof sketch] W.l.o.g. assume $\qo$ is a CQ$^{\neq,\textbf{C}}$. We show we can guess a CQ$^{\neq,\textbf{C}}$ $q' $ from $\mathcal R^-(\qo)$ in polynomial time. Since $\erules$ is linear, any such $q'$ has at most $|\qo|$ atoms, which only share terms from $\qo$. Hence, the length of a rewriting sequence to $q'$ can be bounded by  $|\qo| \times A$, where $A$ is an upper-bound on the number of ``non-isomorphic'' atoms---with isomorphism being the identity on $\terms{\qo}$--- i.e., $A = |P| \times (|\terms{\qo}|+ a)^a$, where $P$ is the set of predicates and $a$ is the maximal arity of a predicate in $P$.  
\end{proof}

This result subsumes previous results establishing $\Pi^{\mathsf P}_2$-membership of $\Sigma$-perfectness verification with UCQs and DL-Lite$_{\mathcal R}$ 
 \cite{Lutz_et_al_KR_2018,Cima2019SemanticCO}. \footnote{We can ignore the disjointness axioms from DL-Lite$_{\mathcal R}$, as they have no impact on the complexity results.}
The theorem actually applies to any FO-rewritable pair $(\qo, \mathcal R)$ such that all the CQs in $\mathcal R^-(\qo)$ can be generated in a polynomial number of rewriting steps.

\section{Computing Minimally Complete and Perfect Abstractions}
\label{sec-min-complete}

\label{sec:min-comp}

Let us first point out that a complete abstraction of a non-Boolean query $\qs$ may not exist, simply because $\mathcal M$ may not transfer all the constants that occur in the answers to $\qs$. This is 
 independent from any target query language. 
 E.g., let $\mathcal M = \{s(x,y) \rightarrow r(y) \}$ and the CQ $q_S(u) = \exists v.s(u,v)$: $q_S$ has no complete  $\KBDM$-abstraction, for any $\KBDM$ with $\map$. Let us characterize when a UCQ$^{\neq,\textbf{C}}$  has a complete $\KBDM$-abstraction: 

\saveContent[Existence of a complete abstraction]{proposition}{prop:existence_comp_trans}{
A CQ$^{\neq\textbf{,C}}$ $\cq_S( \tuple{x})$ has a complete $\KBDM$-abstraction iff for all $x \in \tuple{x}$ there are $m \in \map$ and a homomorphism $h: \body{m} \rightarrow \cq_S(\tuple{x})$ s.t. $x \in h(\fr{m})$. A UCQ$^{\neq,\textbf{C}}$ $\qs(\tuple{x})$ has a complete $\KBDM$-abstraction iff each $\cq_i(\tuple{x}) \in \qs$ has one.
}

Hence, deciding if a non-Boolean UCQ$^{\neq,\textbf{C}}$ has a complete $\Sigma$-abstraction is NP-complete, while it is trivial for a Boolean UCQ$^{\neq,\textbf{C}}$. 

\medskip
\noindent
\textbf{UCQ$^{\neq,\textbf{C}}$ captures perfect $\Sigma$-abstractions of UCQ$^{\neq(\textbf{,C})}$ source queries.}
As already mentioned, chasing a (relevant) UCQ with $\map$ yields a UCQ that is minimally complete within this class.  
However, as illustrated by Ex. \ref{ex-perfect-ucq-dif},  the class UCQ$^{\neq,\textbf{C}}$ may provide a more faithful translation:
the UCQ $\map(\qs)$ is minimally complete within UCQs but not sound, while the UCQ$^{\neq,\textbf{C}}$ $\qo$ is a perfect abstraction.   
We now state the main result of this section: the class UCQ$^{\neq,\textbf{C}}$ captures minimally complete abstractions of source UCQs$^{\neq,\textbf{C}}$, where minimality is w.r.t. \emph{any} ontological query class (still with certain answer semantics)\footnote{See the discussion at the end of this section.}.  

\begin{theorem}[Minimal completeness] \label{th-min-complete} 
For any mapping $\map$ and any UCQ$^{\neq,\textbf{C}}$ $\qs$ that has a complete $\map$-abstraction, there is a UCQ$^{\neq,\textbf{C}}$ $\qo$ such that, for any $\Sigma$ with mapping $\map$,  $\qo$ is a minimally complete $\Sigma$-abstraction of $\qs$. 
\end{theorem}

Note that $\qo$ is also a minimally complete $\map$-abstraction (we take $\KBDM$ with $\erules = \emptyset$). 
However, a minimally complete $\map$-abstraction is generally not a minimally complete $\KBDM$-abstraction, and vice-versa; to obtain Th. \ref{th-min-complete}, we will rely on the specific abstraction computed by the $\mathcal M$-chase.  
Moreover, if there is a perfect $\KBDM$-abstraction of $\qs$, any minimally complete $\KBDM$-abstraction of $\qs$ is perfect, hence UCQ$^{\neq,\textbf{C}}$ also captures perfect abstractions:

\begin{corollary}[Perfectness] For any $\Sigma$ and  any UCQ$^{\neq,\textbf{C}}$ $\qs$, if there is a perfect $\Sigma$-abstraction of $\qs$, then it can be expressed as a UCQ$^{\neq,\textbf{C}}$. 
 \end{corollary}

Furthermore, it is easy to find examples 
in which $\textbf{C}$ or the limited $\neq$ is required to express a perfect abstraction of a UCQ, hence one can argue that UCQ$^{\neq,\textbf{C}}$ is a \emph{minimal} class to express perfect abstractions of UCQs and UCQ$^{\neq,\textbf{C}}$. 

To prove Th.  \ref{th-min-complete}, we first state a fundamental semantic property of OBDA systems.  

\saveContent{proposition}{prop-lemma-3-notes}{ For any databases $D$ and $D'$ on $\PS$, if  $\Sigma(D) \models \mathcal M(D')$, then: \footnote{As regards the formulation of the proposition, note that $\Sigma(D) \models \mathcal M(D')$ is stronger than $\Sigma(D) \models \Sigma(D')$: 
indeed, $\Sigma(D) \models \mathcal M(D')$ implies $ \erules(\Sigma(D)) \models \erules(\mathcal M(D'))$, with 
$\erules(\Sigma(D)) \equiv \Sigma(D)$ and $\erules(\mathcal M(D')) \equiv \Sigma(D')$
}  
\begin{enumerate}
\item  $\modK{\D}{\KBDM} \subseteq \modK{\D'}{\map}$.  
\item Hence: for any $\qo$ on $O$, $\evalC{\qo}{\D',\map}\subseteq \evalC{\qo}{\D,\KBDM}$. 
\end{enumerate} 
}

We now explain how to build the desired abstraction. Only $\map$ is required (not $\erules$), the resulting query being minimally complete for any $\KBDM$ with mapping $\map$. In a nutshell, before chasing $\qs$, we first split each $\cq_i \in \qs$ into an equivalent UCQ$^{\neq,\textbf{C}}$,
whose CQs encode all the ways of mapping $\cq_i$ to a database: 
terms substituted identically are merged, remaining terms are declared distinct ($\neq$) and marked by $\textbf{C}$.   
We will show that chasing the output with $\map$ yields the desired minimally complete $\Sigma$-abstraction. 

Note that a similar split operation is presented in \cite{cima_monotone_2022} to compute minimally complete abstractions expressed in a more complex target language. Such operation is also commonly used to build inverse mappings, see e.g. Ex. \ref{ex:max_recov_of_glav_mapping}. 
For the sake of self-containedness, and to include the processing of constants, we detail our split operation next. 
Given a CQ$^{\neq, \textbf{C}}$ $q$, a partition of $\terms{q}$ is said \emph{admissible} if none of its classes contains two constants nor both terms of a $\neq$-atom from $q$. Informally, each class of the partition gathers the terms of $q$ mapped to the same database constant. 
To each admissible partition  $P_\sigma$ can be assigned a substitution $\sigma$, which is obtained by (1) selecting one term $t_i$ in each class $C_i \in P_\sigma$, with priority given to constants, then to answer variables if any, and (2) setting $\sigma(t_j) = t_i$ for each $t_j \in C_i$ s.t. $t_j \neq t_i$. 
E.g., to $P_\sigma = \{ \{x,u\}, \{v,w,a\}\}$, with $x$ an answer variable and $a$ a constant, is assigned $\sigma = \{ u \mapsto x, v \mapsto a, w \mapsto a \}$.  
 Given a UCQ$^{\neq,\textbf{C}}$ $\qs(\tuple{x})$, 
$\textit{split}(\qs)$ is a UCQ$^{\neq,\textbf{C}}$ built as follows:

\noindent
\rule{\linewidth}{0.2mm}
\hspace*{0.2cm}Let $\textit{split}(\qs) = \emptyset$\\
\hspace*{0.2cm}For each $q_i \in \qs$\\
\hspace*{0.5cm} For each admissible $P_\sigma$ on $\terms{q_i} \cup \consts{\map}$ \\
\hspace*{0.8cm} Let $q' = \sigma(q_i)$  \quad \emph{// $q'$ is consistent} \\
\hspace*{0.8cm} For any $v \in \vars{q'}$  \\
\hspace*{1.1cm} If $v \not \in ~\tuple{x}$ then add $\textbf{C}(v)$ to $q'$ \\
\hspace*{1.1cm}  For any $t \in \terms{q'} \cup \consts{\map}$ with $v \neq t$\\
\hspace*{1.3cm}  Add $v \neq t$ to $q'$\\ 
\hspace*{0.8cm}  Add $q'$ to $\textit{split}(\qs)$\\
\rule{\linewidth}{0.2mm}
Note that $q'$ may include atoms of the form $v \neq c$, where $c$ is a constant from $\map$ that does not occur in \emph{std}$(q')$: this is necessary to ensure the desired behavior of $\textit{split}(\qs)$ (see Lemma  \ref{lemma-split}).  It is easy to check that for any database $D$, $\qs(D) = \textit{split}(\qs)(D)$.

\begin{example}(Minimally complete abstraction) Consider again Example \ref{ex-perfect-ucq-dif} with $\cq_s(u) = \exists v. s_1(u,v)$. 
$$\mathit{split}(\cq_s(u)) = (\exists v. s_1(u,v)
\land \textbf{C}(v) \land u \neq v) \lor s_1(u,u)$$
$$\mathcal \map(\mathit{split}(\cq_s(u))) =  \qo(u)$$
\end{example}

The following lemma states the crucial property of \textit{split$(\qs)$}. We call \emph{grounding} of a CQ$^{\neq,\textbf{C}}$ $q_S$ a substitution $\sigma$ of each variable in $q_S$ by a constant s.t. $\sigma(q_S)$ is consistent. The key point is that any grounding of a $q_s^i \in \textit{split}(\qs)$ is \emph{injective}. It follows that any rule body maps ``in the same way'' to $q_s^i$ and to any of its groundings:

\begin{lemma}\label{lemma-split} Let $\sigma^i_s$ be a grounding of $q_s^i \in \textit{split}(\qs)$. Then: 
$\map( \sigma^i_s(q_s^i)) \equiv \sigma^i_s(\map(q_s^i))$. 
\end{lemma}

\vspace*{-0.2cm}

\begin{proof}[\textbf{Proof of Th. \ref{th-min-complete} (Sketch)}] 
Let $\qo(\tuple{x}) = \map(\textit{split}(\qs))$. 
We show that $\qo$ is a minimally complete $\Sigma$-abstraction of $\qs$, for any $\Sigma$ with mapping $\map$. Completeness follows from the properties of the $\map$-chase. To prove that $\qo$ is minimally complete, we consider any $D$ and $\tuple{c} \in  \eval{\qo}{\Sigma(\D)}$ and show that $\tuple{c}$ is a certain answer to any complete $\Sigma$-abstraction of $\qs$. Let $q^{i}_O \in \qo$ that maps by $h_i$ to $\Sigma(\D)$ with $h_i(\tuple{x}) = \tuple{c}$. Let $q^{i}_S \in \textit{split}(\qs)$ such that $q^{i}_O = \map(q^{i}_S)$. Let $D_i$ be  obtained by a grounding of $h_i(q^{i}_S)$. Since $h_i(\tuple{x}) = \tuple{c}$, $\tuple{c} \in \eval{q^{i}_S}{D_i}$. Since $q^{i}_O$ is a complete $\Sigma$-abstraction of $q^{i}_S$, $\tuple{c} \in \eval{q^{i}_O}{\Sigma(D_i)}$. We have $\map(h_i(q^{i}_S)) \equiv \map(D_i)$, hence, from Lemma \ref{lemma-split}, $\map(D_i)$ maps to $h_i(q^{i}_O)$. Since $h_i(q^{i}_O) \subseteq \Sigma(D)$,  $\map(D_i)$ maps to $\Sigma(D)$. 
So, by Prop. \ref{prop-lemma-3-notes}, for all ontological query $Q$, $\evalC{Q}{D_i, \Sigma} \subseteq \evalC{Q}{D, \Sigma}$, hence if $Q$ is $\Sigma$-complete then $\tuple{c} \in \evalC{Q}{D, \Sigma}$.  
\end{proof}

\noindent \textbf{Complexity of expressibility}
With these results in hand, we can now study the complexity of determining whether a UCQ$^{\neq,\textbf{C}}$ $\qs$ has a perfect $\map$-abstraction. Let us say that $\map$ is \emph{frontier-bounded} if the frontier of all its rules is bounded by a constant. 
From \cite{Lutz_et_al_KR_2018} we know that UCQ $\map$-expressibility is $\Pi^{\mathsf P}_2$-complete in a GAV setting with bounded predicates (in rule heads). We observe that $\Pi^{\mathsf P}_2$-membership can be extended to GLAV mappings with unbounded predicate arity provided that $\map$ is frontier-bounded. Indeed, for a CQ $q_{s{_i}} \in Q_S$, $\map(q_{s{_i}})$ is built from rule heads whose frontier is substituted by $\terms{q_{s{_i}}}$; hence, for each $m \in \map$, the number of substitutions that need to be considered is bounded by $\terms{q_{s{_i}}}^{|\fr{m}|}$. Therefore, when $\map$ is frontier-bounded, we can build each $\map(q_{s{_i}})$ by making a polynomial number of calls to an NP oracle, asking for each $m \in \map$ with $\fr{m} = (x_1, \ldots, x_k)$ (according to an arbitrary total ordering of the frontier variables) and tuple $t = (y_1, \ldots, y_k) \in \terms{q_{s{_i}}}^k$ if there is a homomorphism $h: \body{m} \rightarrow q_{s{_i}}$ such that $h(\fr{m}) = t$. 
When $\map$ is not frontier-bounded, $\map(Q_S)$ can be computed in ExpTime, which yields a Co-NExpTime upper bound.
These arguments can be generalized to (unrestricted) $\map$-expressibility of UCQ$^{\neq,\textbf{C}}$ queries, as shown next. 

\begin{theorem}(Complexity of $\map$-expressibility)  UCQ$^{\neq,\textbf{C}}$ $\mathcal M$-expressiblity is $\Pi^{\mathsf P}_2$-complete when $\map$ is frontier-bounded, otherwise it is in Co-NExpTime.
\end{theorem} 

\begin{proof}[Proof sketch] To check that  $Q_s$ is \emph{not} $\map$-expressible, 
one can guess a CQ$^{\neq}$ from $Q_s$ and a partition on its terms,
 which yields a CQ$^{\neq}$ $q_{s{_i}}$ from split($Q_s$), 
 compute $q_{o{_i}}=\map({q_{s{_i}}})$ 
 and guess  a rewriting $q'_{s{_i}}$ of $q_{o{_i}}$ (of polynomial size in $q_{o{_i}}$) 
 such that $q'_{s{_i}} \not \sqsubseteq Q_s$ (test in $\Sigma^P_2$).
 If $\map$ is frontier-bounded,  $q_{o{_i}}=\map({q_{s{_i}}})$ can be built by making a polynomial number of calls to an NP oracle, otherwise, it can be computed in ExpTime. Hence, the (co-)problem is in $\Sigma^P_2$ if $\map$ is frontier-bounded, otherwise in NExpTime. 
 $\Pi^{\mathsf P}_2$-hardness follows from \cite{Lutz_et_al_KR_2018}. 
\end{proof}

\noindent \textbf{Discussion on related frameworks} 
We will now  discuss our framework further in relationship with previous work.

As shown above, one can decide if a UCQ$^{\neq,\textbf{C}}$  $\qs$ has a perfect $\map$-abstraction by simply checking if $\q_O = \map(\textit{split}(\qs))$ is a sound $\map$-abstraction,
i.e., $\map^-(\q_O) \sqsubseteq \q_S$. This may seem contradictory with other results from the literature. In particular, it is shown in \cite{DBLP:conf/lics/CimaCLP21} that determining if a CQ has a perfect $\map$-abstraction is undecidable.
In fact, the crucial point is the semantics of ontological queries. We consider the widely adopted semantics of certain answers. As a consequence, ontological queries are necessarily monotone, in the following sense:  $Q_O$ is \emph{monotone} if for all $D_1,D_2$ on $S$, 
if $\modK{\D_2}{\KBDM} \subseteq \modK{\D_1}{\KBDM}$ then any answer to $Q_O$ on $(D_1, \KBDM)$ is an answer to $Q_O$ on $(D_2, \KBDM)$. 
 This is a corollary of our Prop. \ref{prop-lemma-3-notes}.  A more general notion of ontological query is investigated in the above-mentionned paper, which allows for non-monotone queries. Note that a source query may have no perfect monotone abstraction but a perfect abstraction in this more general setting, which is studied in \cite{Cima2020NonMonotonicOA}.

\smallskip
The $\textbf{C}$ predicate was introduced in \cite{fagin2008quasi}, under the name \emph{is-constant}, to define specific kinds of inverses of GLAV mappings, which are disjunctive$^{\neq,\textbf{C}}$ mappings (Definition \ref{def-disj-mapping}). It has been commonly used since then in the data exchange litterature, not only in inverse mappings but also in queries, see e.g., \cite{Arenas2008TheRO}. Similar notions have also been studied in KR, as closed-world variables \cite{DBLP:conf/ijcai/AmendolaLMV18}, or nominal variables in description logics \cite{DBLP:conf/kr/KrotzschR14}. 
 We think that $\textbf{C}$ yields a very simple and effective way of dealing with unknown values introduced by mappings (and ontologies). First, it is easy to understand for a user and its introduction has no impact on computational complexity. Second, it can be handled using off-the-shelf tools. Indeed, one can slighly modify the mapping by adding a head atom $\textbf{C}(t)$ for each frontier variable or constant $t$: then, $\textbf{C}$-atoms in ontological queries can be processed just like the standard atoms. 
 A more general way of introducing some closed world reasoning would have been to extend queries with the epistemic operator $\textbf{K}$   (``is known''), as in  \cite{Cima2020NonMonotonicOA,cima_monotone_2022}. However, this operator does not have the simplicity of $\textbf{C}$, which seems a better choice to us in the context of  abstractions under standard certain answer semantics. 
 On the other hand, $\textbf{K}$ comes into its own in the context of \emph{non-monotone} ontological queries, which inherently requires a more general semantics than certain answers.

\vspace*{-0.1cm}
\section{Computing Maximally Sound Abstractions}\label{sec:max_sound}
\label{sec-max-sound}

A (U)CQ$^{(\neq)}$ always has a sound abstraction (the empty UCQ, which has no answer) but may not have a \emph{maximally sound} abstraction expressible as a UCQ$^{(\neq, \textbf{C})}$. Finding a suitable language for such abstractions remains open. In this section, we make progress by providing a characterization of maximally sound $\map$-abstractions of UCQ$^{(\neq)}$s, which we further extend to $\KBDM$-abstractions with \emph{fus} rules. For that, we rely on specific inverse mappings from $\PO$ to $\PS$, which, as explained later, correspond to so-called \emph{maximum recoveries}.  Such mappings have disjunctive heads, as defined next. 

\begin{definition} [Disjunctive$^{\neq, \textbf{C}}$ mapping]\label{def-disj-mapping} 
A \emph{disjunctive} (resp. \emph{disjunctive$^{\neq, \textbf{C}}$}) \emph{mapping} from a source schema $\source$ to a target schema $\target$ is a set of $\source$-to-$\target$ disjunctive existential rules of the form 
$\forall\tuple{x}.(\exists \vect{y}. B[\vect{x},\tuple{y}]) \rightarrow \bigvee\limits^n_{i=1} \exists \vect{z_i}. H_i[\vect{x}, \vect{z_i}]$, where $B$ is a $CQ$ (resp. a $CQ^{\neq, \textbf{C}}$) and each $H_i$ is a CQ, all with answer variables  $\vect{x}$. 
\end{definition} 

Given a (GLAV) mapping $\map$ from $S$ to $O$, we will consider a disjunctive$^{\neq, \textbf{C}}$ mapping $\map_\lor$ from $O$ to $S$, which has the property of being a maximum recovery of $\mathcal M$  \cite{Arenas2009TheRO,Arenas2009InvertingSM}. Before entering into the formal framework of maximum recoveries, we first explain how $\map_\lor$ is built. Briefly, each rule of $\map_\lor$ is obtained by rewriting a rule head from $\map$ against $\map$. Precisely, for each $m\in \map$ with head  $\exists \vect y. H[\vect x, \vect y]$ (seen as a CQ), $\map_\lor$ has the following disjunctive rule:  

\vspace*{-0.3cm}
$$\forall\tuple{x}.(\exists \vect y. H[\vect x, \vect y] \land \textbf{C} [\vect x]) \rightarrow \map^-({\exists \vect y. H[\vect x, \vect y]})$$ 

This is illustrated in Ex. \ref{ex:max_recov_of_glav_mapping}. As shown in the example, the head of the obtained rule may contain equalities; these equalities can be turned into inequalities in the rule body, by a split operation similar in spirit to that described in Sect. \ref{sec-min-complete}, which yields a rule complying with  Def. \ref{def-disj-mapping}; see \cite{Arenas2009InvertingSM} for details. 

\begin{example}\label{ex:max_recov_of_glav_mapping}Let $\map = $\\
$\begin{cases} 
    m_1 = s_1(x)   \rightarrow \exists y. p(x,y) &      m_3 = s_3(x,y)   \rightarrow r(x,y)\\
    m_2 = s_2(x,y)   \rightarrow p(x,y) &    m_4 = s_4(x)  \rightarrow r(x,x)  
    \end{cases}$
    \\
By rewriting the CQs $\head{m_i}(\fr{m_i})$, one gets $\map'$=
\\ 
 $\begin{cases}
    m_1' = p(x,y) \land \textbf{C}(x) &\hspace{-0.2cm}\rightarrow s_1(x) \lor \exists z. s_2(x,z) \\
    m_2' = p(x,y) \land \textbf{C}(x) \land \textbf{C}(y) &\hspace{-0.2cm}\rightarrow s_2(x,y) \\
   m_3' = r(x,y) \land \textbf{C}(x) \land \textbf{C}(y) &\hspace{-0.2cm}\rightarrow s_3(x,y) ~\lor 
   \\ & \hspace{0.2cm}(s_4(x) \land x=y) \\
    m_4' = r(x,x) \land \textbf{C}(x) &\hspace{-0.2cm}\rightarrow s_3(x,x) \lor s_4(x)
 \end{cases}$
\\
Moreover, Rule $m_3'$ with equality can be replaced by two rules obtained by considering that, in $\body{m_3'}$, either \mbox{$x=y$} (which yields $m_4'$, already present) or $x\neq y$, which yields: 
$m_3'' = r(x,y) \land \textbf{C}(x) \land \textbf{C}(y)\land x \neq y \rightarrow s_3(x,y)$ 
\\
Finally, $\map_\lor = \{m_1', m_2', m_3'', m_4'\}$. 
\end{example}

We furthermore consider the rewriting operator for UCQs against \emph{disjunctive} mappings introduced in \cite{Leclre2023QueryRW}. This operator is sound and complete and yields a possibly infinite disjunction of CQs. Its extension to UCQ$^{\neq,\textbf{C}}$s and disjunctive$^{\neq, \textbf{C}}$ mappings is straightforward. We can now outline our characterization of maximal sound $\map$-abstractions:  For any mapping $\map$ (from $\PS$ to $\PO$) and UCQ$^{\neq,\textbf{C}}$ $\ucq_S$ on $\PS$,
let $\map_{\lor}$ be the disjunctive$^{\neq, \textbf{C}}$ mapping from $\PO$ to $\PS$ built as above; then, $\map_\lor^-(\qs)$, i.e., the rewriting of $\qs$ against  $\map_{\lor}$, is a maximally sound abstraction of $\qs$. This result relies on the fact that  $\map_{\lor}$ is a maximum recovery of $\map$ and is proven in Th. \ref{th-max-recovery}. 

\smallskip
\noindent
\textbf{Maximum recoveries} 
The following definitions and results come from \cite{Arenas2009TheRO,Arenas2009InvertingSM}. 
The notion of a maximum recovery is defined on \emph{abstract mappings}, which may then be \emph{specified} by \emph{concrete} mappings, i.e.,  provided with a specific syntax (e.g., GLAV). 
An \emph{abstract mapping}  $\mapa$  from a schema $\source$ to a schema $\target$ is any relation from the $\source$-instances to the $\target$-instances.\footnote{In the cited work, instances are finite, but the definitions work in the infinite case.}
Let $\q_\target$ be a query on $\target$. 
Given an $\source$-instance $I$, we denote by $\certain{\mapa}{\q_\target}{I} = \bigcap_{(I,J) \in \mapa} \q_\target(J)$ the \emph{certain answers} to $\q_\target$  through $I$ and $\mapa$. 
 A query $\q_\source$ on $\source$ such that $\q_\source(I) = \certain{\mapa}{\q_\target}{I}$ for all instance $I$ on $\source$, is called a \emph{perfect rewriting} of $\q_\target$ through $\mapa$ (such $\q_\source$ may not exist).
The composition of two abstract mappings $\mapa$ and $\mapa'$ is denoted by $\mapa \bullet \mapa'$.\footnote{We use $\bullet$ to avoid confusion with the classical $\circ$: $\mapa \bullet \mapa'$ can be read $\mapa' \circ \mapa$.} 
Given an abstract mapping $\mapa$ from $\source$ to $\target$, the abstract mapping $\rmap$ from $\target$ to $\source$ is a \emph{recovery} of  $\mapa$ if for any query $\q_\source$ on $\source$ and instance $I$ on $\source$, $\certain{\mapa \bullet \rmap}{\q_\source}{I} \subseteq \q_\source(I)$; and $\rmap$ is a \emph{maximum} recovery if, moreover, for any recovery $\rmap'$ of $\mapa$, $\certain{\mapa \bullet \rmap'}{\q_\source}{I} \subseteq \certain{\mapa \bullet \rmap}{\q_\source}{I}$.

An abstract mapping $\mapa$ from $\source$ to $\target$ is \emph{specified} by a (GLAV) mapping $\map$ from $\source$ to $\target$ if: for every pair of $(\source,\target)$-instances $(I,J)$, $(I,J) \in \mapa$ iff $I \cup J \models \map$.
 In this case, $\certain{\mapa}{\q_\target}{I} = \evalC{\q_\target}{\I,\map}$ holds for any query $\q_\target$ and, when $\q_\target$ is a UCQ$^{\neq,\textbf{C}}$, $\mrew{\q_\target}$ is a perfect rewriting of $\q_\target$ through $\mapa$. 
 Not all abstract mappings have a maximum recovery. However,  
when the  source instances are ground,  
a GLAV mapping always has one, taking the form of a disjunctive$^{\neq, \textbf{C}}$ mapping (Def. \ref{def-disj-mapping}).\footnote{More precisely: For every GLAV mapping $\map$, which specifies an abstract mapping $\mapa$, there is a concrete mapping $\map_\lor$ that specifies a  maximum recovery of  $\mapa$ and can be expressed as a disjunctive mapping$^{\neq, \textbf{C}}$. For the sake of simplicity, we say that $\map_\lor$ is a (concrete) maximum recovery of $\map$. }

\smallskip
\noindent \textbf{Maximally sound $\map$-abstractions} 
Let $\mapa$ be an abstract mapping specified by a GLAV mapping $\map$. 
The following lemma shows that a perfect rewriting of a source query $\q_\source$ through a maximum recovery of $\mapa$ behaves similarly to a maximally sound $\source$-to-$\target$ translation of $\q_\source$ through $\map$ (i.e., an $\map$-abstraction of $\q_\source$ when $\source = S$ and $\target = O$).  Indeed, Point (1) corresponds to the soundness of an $\map$-abstraction and Point (2) to its maximality.

\saveContent[]{lemma}{lem:link_max_recov_max_sound_trans}{
Let  $\q_\source$ be a query on $\source$, $\mapa$ be an abstract mapping from $\source$ to $\target$ that has a maximum recovery $\rmap$. Let $\q_\target$ be a perfect rewriting of $\q_\source$ through $\rmap$. 
Then, for any $\source$-instance $I$: (1) $\certain{\mapa}{\q_\target}{I} \subseteq \q_\source(I)$, and (2) $\certain{\mapa}{\q'_\target}{I} \subseteq \certain{\mapa}{\q_\target}{I}$ for any query $\q'_\target$ such that $\certain{\mapa}{\q'_\target}{I} \subseteq \q_\source(I)$.
}

Finally, Th. \ref{th-max-recovery} directly relies on Lemma \ref{lem:link_max_recov_max_sound_trans}:

\saveContent[]{theorem}{th-max-recovery}{
    Let $\map$ be a (GLAV) mapping from $S$ to $O$, $\map_\lor$ be a disjunctive$^{\neq,\textbf{C}}$ mapping that is a (concrete) maximum recovery of $\map$, and $\q_S$ be a UCQ$^{\neq,\textbf{C}}$ on $S$. Then $\Srewriting{\q_S}{\map_\lor}$ is a maximally sound $\map$-abstraction of $\q_S$.
}

In general, $\Srewriting{\q_S}{\map_\lor}$ is a possibly infinite disjunction of UCQ$^{\neq,\textbf{C}}$. Yet, the next proposition gives cases where  it is a UCQ$^{\neq, \textbf{C}}$.
In such cases, the rewriting algorithm from \cite{Leclre2023QueryRW} can be used to effectively output a maximally sound $\map$-abstraction. 

\saveContent[]{proposition}{thm:maximally_sound_UCQ_abstraction_we_can_compute}{
   The maximally sound $\map$-abstraction of a UCQ$^{\neq,\textbf{C}}$ $\ucq$ is a UCQ$^{\neq,\textbf{C}}$ when:
    \begin{enumerate}
        \item      $\mrew{\head{m}}$ is a CQ$^{\neq, \textbf{C}}$ for all $m\in \map$; or: 
        \item $\ucq$ contains only full CQs$^{\neq, \textbf{C}}$ (i.e., without existential variables) and $\map$ is GAV; or: 
        \item $\ucq$ contains only  atomic CQs$^{\neq, \textbf{C}}$ (i.e., with at most one standard atom). 
    \end{enumerate}
}

\vspace*{-0.2cm}
\begin{proof} Let $\map_{\lor}$ be a maximum recovery of $\map$. 
    (1) $\map_{\lor}$ is a conjunctive mapping. 
    (2)    All rules in $\map_{\lor}$  are lossless (all body variables are frontier) which guarantees to get a UCQ$^{\neq, \textbf{C}}$-rewriting from any full CQ$^{\neq, \textbf{C}}$.  (3) Disjunctive source-to-target rule sets guarantee to get a UCQ$^{\neq, \textbf{C}}$-rewriting from any atomic CQ$^{\neq, \textbf{C}}$.
    Points (2) and (3) follow from the rewriting algorithm in  \cite{Leclre2023QueryRW}.  
\end{proof}

\noindent
\textbf{Note.} We remark that Th. \ref{th-max-recovery} contradicts a result from  \cite{Arenas2009InvertingSM} (see Th. 4.4).
This result states that  a recovery that maximally recovers answers to CQs (not UCQs),  called a CQ-maximum recovery, can be specified by a conjunctive mapping (i.e., without disjunctive heads). 
But then, the rewriting of a CQ through a CQ-maximum recovery would always be finite (this is a property of conjunctive mappings), hence a CQ would always have a maximally sound $\map$-abstraction as a UCQ$^{\neq, \textbf{C}}$, which is false. 

\smallskip
\noindent \textbf{Maximally sound $\KBDM$-abstractions}
 We now extend previous results to an OBDA specification with a \emph{fus} ontology $\erules$. 
 A suitable disjunctive$^{\neq, \textbf{C}}$ mapping from $\PO$ to $\PS$, say $\map_\lor^{\Sigma}$, is obtained by rewriting each rule head of $\map \cup \erules$ against $\map \cup \erules$. For a rule
 head $\exists \vect{y}. H[\vect x, \vect y]$, this yields the disjunctive rule $\forall \vect{x}. (\exists \vect{y}. H[\vect x, \vect y] \land \textbf{C} [\vect x]) \rightarrow \Srewriting{\exists \vect y. H[\vect x, \vect y]}{\Sigma}$. 
 To bring OBDA specifications into the maximum recovery framework, we say that an abstract 
mapping $\map_\mathcal{A}$ from $\PS$ to $\PO$ is \emph{specified} by 
$\KBDM=\kbdm$  if, for all $\PS$-database $D$ and $\PO$-instance $J$, $(D,J) \in \map_\mathcal{A}$ iff $D \cup J \models \map$ and $J \models \erules$ both hold. 

\saveContent[]{theorem}{th:specification_max_recovery_of_obda_specification}{
    Let  $\KBDM$ be an OBDA specification with FO-rewritable $\erules$. Then:
    \begin{enumerate}
        \item $\map_\lor^{\Sigma}$ is a (concrete) maximum recovery of $\KBDM$.
        \item For any UCQ$^{\neq,\textbf{C}}$ $\q_S$ on $S$, $\Srewriting{\q_S}{\map_\lor^{\Sigma}}$ is a maximally sound $\KBDM$-abstraction of $\q_S$.
    \end{enumerate}
}

\vspace*{-0.3cm}
\begin{proof}[Proof (sketch)]
(1)
    Since $\map \cup \erules$ is fus, $\map_\lor^{\Sigma}$ is well defined. We first prove that $\map_\lor^{\Sigma}$ specifies a recovery $\Sigma'_{\mathcal{A}}$ of the abstract mapping ${\Sigma_{\mathcal{A}}}$ specified by $\Sigma$. To do that, we  prove that for all $\PS$-databases $D$, there is an $\PO$-instance $J$ s.t. $(D,J) \in \Sigma_{\mathcal{A}}$ and $(J,D) \in \Sigma'_{\mathcal{A}}$. Such $J$ always exists, f.i. $J =  \Sigma(D)$. Then, we prove that $\Sigma'_{\mathcal{A}}$ is a maximum recovery of $\Sigma_{\mathcal{A}}$, using Prop. 3.8 from \cite{Arenas2009TheRO}, from which follows that $\Sigma'_{\mathcal{A}}$ is a maximum recovery of $\Sigma_{\mathcal{A}}$ iff $\Sigma'_{\mathcal{A}}$ is a recovery and for every $(D_1, D_2) \in \Sigma_{\mathcal{A}} \bullet \Sigma'_{\mathcal{A}}$, it is the case that $\emptyset \neq \sol{\Sigma_{\mathcal{A}}}{D_2} \subseteq \sol{\Sigma_{\mathcal{A}}}{D_1}$. $\Sigma'_{\mathcal{A}}$ has this property by construction of $\map_\lor^{\Sigma}$.
 (2) The proof is similar to the proof of Th. \ref{th-max-recovery}, using Point (1) and Lemma \ref{lem:link_max_recov_max_sound_trans}.
\end{proof}

\vspace*{-0.2cm}
Prop. \ref{thm:maximally_sound_UCQ_abstraction_we_can_compute} can be extended to $\Sigma$-abstractions  as follows: (1) taking $\map \cup \erules$ instead of $\map$; (2) and (3): taking rule classes ensuring  that $\erules^-(Q)$ has the desired property, in particular lossless rules for (2) and 
linear rules for (3).

\section{Conclusion}
\label{conclusion}

We have investigated the properties of the query class UCQ$^{\neq,\textbf{C}}$ for capturing abstractions in an OBDA setting under certain answer semantics. We found that this class enjoys nice computational behavior in this context. We proved that it is able to express any minimally complete—and therefore any perfect—abstraction of a source UCQ$^{\neq,\textbf{C}}$, when such an abstraction exists. Although  a maximally sound abstraction of a UCQ always exists, it may not be expressible in UCQ$^{\neq,\textbf{C}}$. However, we identified an interesting connection with the notion of maximum recovery from data exchange, and showed that a maximally sound $\map$-abstraction of a source UCQ$^{\neq,\textbf{C}}$ is precisely its rewriting with a maximum recovery of  $\map$. While the ontology plays no role in minimal completeness, it does in maximal soundness. Accordingly, we extended the preceding result to OBDA specifications with \emph{fus} ontologies. 

Among the open questions, it remains unknown whether the problem of determining if a (U)CQ admits a maximally sound abstraction in UCQ$^{\neq,\textbf{C}}$ is decidable. Moreover, no known algorithm is guaranteed to terminate whenever such a finite abstraction exists.

\section*{Acknowledgements}
We thank the reviewers for their helpful comments.

\bibliographystyle{kr}
\bibliography{biblio-kr-25}

\onecolumn
\newpage
\appendix
\section*{Appendix}

\medskip
This appendix contains detailed proofs of the paper's results, as well as further examples, discussions and results that could not be included in the paper due to space constraints.  

\section{Complements to Section \ref{sec-ucq-dif}}

The next example shows that homomorphism is not a necessary condition for containment of CQ$^{\neq,\textbf{C}}$s. 

\begin{example}[Query containment within UCQ$^{\neq,\textbf{C}}$]
Consider the following Boolean queries:
\\
$\cq_1 = \exists u,v. ~p(u,v) \land \textbf{C}(u)  \land \textbf{C}(v) \land u \neq v$
\\
$\cq_2 = \exists x,y,z.~p(x,y) \land p(x,z) \land \textbf{C}(x)  \land \textbf{C}(y)   \land \textbf{C}(z) \land y \neq z$.
\\
Note that all the variables in these queries occur in a \textbf{C}-atom, hence inequalities are allowed between all (distinct) terms.  
 There is no homomorphism from $\cq_1$ to $\cq_2$; however $\cq_2 \sqsubseteq \cq_1$. 
Indeed, for any database $\D$ that answers yes to $\cq_2$, either $x$ and $y$ are mapped to distinct constants, and $\D$  answers 
yes to $\cq_1$, or  $x$ and $y$ are mapped to the same constant, in which case $x$ and $z$ are necessarily mapped to distinct constants (because of the atom $y \neq z$), and $\D$ answers yes to $\cq_1$. 
\end{example}

\medskip

  Let \emph{UCQ$^{\neq,\textbf{C}}$ containment} be the problem that takes as input two queries $Q_0$ and $Q_1$ in UCQ$^{\neq,\textbf{C}}$, and asks if $Q_0 \sqsubseteq Q_1$. 
  This problem is known to be  $\Pi^{\mathsf P}_2$-complete, already when both queries are in CQ$^{\neq,\textbf{C}}$  
\cite{DBLP:journals/jcss/Meyden97,DBLP:conf/pods/KolaitisMT98}. 
As a complementary result, we prove that  $\Pi^{\mathsf P}_2$-hardness already holds when $Q_0$ is Boolean CQJFE, which is a very simple kind of CQ. 

\medskip
\textbf{Theorem \ref{th-QC}} 
\emph{
 The UCQ$^{\neq,\textbf{C}}$ containment 
 problem is $\Pi^{\mathsf P}_2$-hard when  $Q_0$ is a Boolean CQJFE and $Q_1$ is a Boolean CQ$^{\neq,\textbf{C}}$. 
}

\medskip
To prove this result, we adapt a reduction from  (Abiteboul et al., 1991)\footnote{Serge Abiteboul, Paris C. Kanellakis, Gösta Grahne:
On the Representation and Querying of Sets of Possible Worlds. Theor. Comput. Sci. 78(1): 158-187 (1991)}. In that paper, they study query containment given databases with incomplete information. An instance is a complete database, and a query is defined as a mapping from a set of instances to a set of instances. An incomplete database represents a set of instances, and it is defined as a set of relations featuring null values, and furthermore provided with equalities and inequalities.  The following problem, denoted by CONT($q_0,q)$, is defined as follows: $q_0$ and $q$ are fixed queries; given incomplete databases $\phi_0$ and $\phi$, does it hold that $q_0(\phi_0) \subseteq q(\phi)$? In particular CONT(\_,\_) is the case where $q_0$ and $q$ are the identity, hence the question is whether the set of instances represented by $\phi_0$ is included in the set of instances represented by $\phi$. Their theorem 4.2 proves that CONT(\_,\_) is $\Pi_2^P$-complete for data complexity for several kinds of incomplete databases. Our proof is inspired from their proof of Case (1), where we turn incomplete databases into CQ$^{\neq,\textbf{C}}$s. 

\begin{proof} To prove $\Pi^{\mathsf P}_2$-hardness, we reduce the $\Pi^{\mathsf P}_2$-complete problem $\forall\exists$3CNF, which takes as input a quantified Boolean formula $\forall X.\exists~Y.\phi[X,Y]$, where $X$ and $Y$ are sets of variables, $X \cap Y = \emptyset$ and $\phi$ is a 3CNF, and asks if for all truth assignments of variables in $X$, there is a truth assignment of variables in $Y$ that make $\phi$ true. 

\medskip
Let $\forall X\exists Y~C$ be an instance of $\forall\exists3CNF$. Let $X = \{x_1, \ldots, x_n\}$. We consider the following set of predicates: for each $i \in \{1,\ldots,n\}$, there is a binary predicate $val_i$; there is also a ternary predicate $cl$ ($cl$ for ``clause'').
We build $Q_0$ and $Q_1$ such that $Q_0$ is a Boolean CQJFE and $Q_1$ is a CQ$^{\neq}$. 

\medskip
We build $Q_0$ as follows. For each $i \in \{1,\ldots,n\}$, there is a variable $z_i$. 
The atoms are the following:
\begin{enumerate}
\item For each $i \in \{1,\ldots,n\}$, two atoms: $val_i(0,z_i)$ and $val_i(1,0)$, where $z_i$ is a variable.
\item All the $cl$-atoms that make a 3-clause true, i.e., all the $cl$-atoms on $\{0,1\}^3$ except $(0,0,0)$. 
\end{enumerate}
Note that, to obtain a ground instantiation of $Q_0$, we only have to instantiate the $z_i$. 

\medskip
We build $Q_1$ as follows. For each propositional variable $x \in X \cup Y$, there are two variables: $x$ and $\bar x$. 
Furthermore, for each $i \in \{1,\ldots,n\}$, there are four variables: $u_i$, $w_i$, $v_i$ and $y_i$. Finally, there are two constants, arbitrarily denoted by 5 and 6. The atoms are the following:

\begin{enumerate}
\item For each $i \in \{1,\ldots,n\}$, two atoms: $val_i(u_i,w_i)$ and $val_i(v_i,y_i)$, where all the terms are variables.
\item For each clause $(l_1 \lor l_2 \lor l_3) \in C$, the atom $cl(var(l_1),var(l_2),var(l_3))$, where $var(l_i) = l_i$ if $l_i$ is positive, and $\bar l_i$ if $l_i$ is negative. 
\item For each propositional variable $x \in X \cup Y$, the atom $x \neq \bar x$.
\item For each propositional variable $x_i \in X$, the atoms $x_i \neq v_i$, $\bar x_i \neq u_i$, $w_i \neq 5$ and $y_i \neq 6$. 
\end{enumerate}

\emph{Intuition:} (1) will be used to consider all the assignements of propositional variables $x_i \in X$. The idea is that when the propositional variable $x_i$ is true, $u_i$ is mapped to $1$ in $Q_0$. When it is false, $v_i$ is mapped to $1$ in $Q_0$. (2) encodes $C$ and (3) encodes the fact that $x$ and $\bar x$ cannot have the same value. The inequalities in (4) ensure the consistency of values taken by the variables $x_i$ and $\bar x_i$.  Constants 5 and 6 are arbitrary, they just need to be different from 0 and 1. 

\medskip
($\Leftarrow$) We prove that, if $(Q_0,Q_1)$ is a positive instance of query containment then $\forall X\exists Y~C$ is a positive instance of  $\forall\exists3CNF$. 

Assume that $Q_0 \sqsubseteq Q_1$, i.e., every instantiation of $Q_0$ is a model of $Q_1$. Consider any truth valuation of the symbols in $X$. Then, consider an instantiation $M$ of $Q_0$ such that: for every $z_i$ in $Q_0$, $z_i \mapsto 5$ if $x_i$ is true, otherwise $z_i \mapsto 6$. Since $M$ is a model of $Q_1$ (by hypothesis), $u_i \mapsto 1$ if $x_i$ is true (indeed, since $w_i \neq 5$, we have $w_i \mapsto 0$ and $u_i \mapsto 1$ - and for other symbols: $\bar x_i \mapsto 0$, $x_i \mapsto 1$, $v_i \mapsto 0$, $y_i \mapsto 5$), otherwise $v_i \mapsto 1$ (indeed, since $y_i \neq 6$, we have $y_i \mapsto 0$ and $v_i \mapsto 1$ - and for other symbols: $x_i \mapsto 0$, $\bar x_i \mapsto 1$, $u_i \mapsto 0$, $w_i \mapsto 6$). 
Since $M$ is a model of $Q_1$, the $cl$-atoms are mapped to the instantiated $cl$-atoms in $Q_0$ while satisfying the inequalities, which yields a truth valuation of the symbols in $X \cup Y$ that satisfies $C$. We conclude that any any truth valuation of the symbols in $X$ can be extended to a truth valuation of the symbols in $Y$ that satisfies $C$. 

\medskip
($\Rightarrow$) We prove that if $\forall X\exists~Y C$ is a positive instance of  $\forall\exists3CNF$ then $(Q_0,Q_1)$ is a positive instance of query containment. 
Assume $\forall X~\exists~Y~C$ is a positive instance of  $\forall\exists~3CNF$. Let $M$ be any instantiation of $Q_0$. Consider the following truth assignment of $X$: 
let $x_i \mapsto true$ if $z_i$ is instantiated by $5$, otherwise $x_i \mapsto false$. By hypothesis, this truth assignment can be extended to a truth assignment of $Y$ that satisfies $C$. We consider the following assignment of the variables of $Q_1$, which shows that $M$ is a model of $Q_1$:
\begin{itemize}
\item For each $i \in \{1,\ldots n \}$: if $z_i$ is instantiated by $5$, then $w_i \mapsto 0$, $u_i \mapsto 1$, $v_i \mapsto 0$, $y_i \mapsto 5$; 
otherwise, $v_i \mapsto 1$, $y_i \mapsto 0$, $u_i \mapsto 0$, $w_i \mapsto c$, where $c$ is the instantiation of $z_i$. 
\item The other variables are mapped to the truth assignement of the corresponding literal. 
\end{itemize}

Then, we obtain a homomorphism from the standard atoms of $Q_1$ (i.e., the $val_i$- and $cl$-atoms) to $M$, and we check that the $\neq$-atoms are satisfied.

\end{proof}


\section{Complements to Section \ref{sec:min-comp}}

\medskip
\paragraph{Remark on the existence of a complete abstraction.} To ensure the existence of complete $\KBDM$-abstractions for all queries, one may enrich the mapping with rules of the form $p(x_1,\ldots,x_n) \rightarrow \top(x_1) \land \ldots \land \top(x_n)$ for all the source predicates $p_{/n}$, see e.g., \cite{Cima2019SemanticCO}. However, this amounts to transferring \emph{all} data values 
to the ontological level, whereas the role of a mapping is precisely to select relevant data.  Moreover, this may lead to very unintuitive minimally complete abstractions, as illustrated by the next example.

\begin{example}
Consider an OBDA specification with source schema $\PS = \{s\text{-}cat(\cdot), s\text{-}owner(\cdot, \cdot)\}$ about cats and their owners, and a mapping 
$\map = \{s\text{-}cat(x) \rightarrow cat(x)\}$.  
Let $\cq_\PS(x) = \exists y. s\text{-}owner(x,y) \land s\text{-}cat(y)$ asking for the owners of cats. There is no complete abstraction of $\cq_\PS$ because owners are not transferred to the ontological level. If we add to $\map$ the rule $s\text{-}owner(x,y) \rightarrow \top(x) \land \top(y)$, then  $\map(\cq_\PS) =   \top(x)\land \exists y. Cat(y)$ is a minimally complete abstraction of $\cq_\PS$, which, by the semantics of $\top$, retrieves all the database values as soon as the data mentions a cat.
\end{example}

For practical use, it seems preferable to identify and exclude the queries that do not admit a complete $\map$-abstraction, because such queries are intrinsically not relevant to abstraction through $\map$. 

\recallContent[Proposition]{prop:existence_comp_trans}

\begin{proof}
$~$
    \begin{itemize}
        \item  For a CQ$^{\neq,\textbf{C}}$ $\cq_\PS(\vect{x})$:
        
        $(\Rightarrow)$
        Let $g$ be an injective substitution of the variables in $\cq_\PS$ by fresh constants (i.e., that do not occur in $\cq_\PS$), and let $D$ be the obtained $S$-database, i.e., $D = g(\cq_\PS)$.  
        In particular, $g(\vect{x}) \in \cq_\PS(D)$. Hence, for every complete $\KBDM$-abstraction $\qo$ of $\cq_\PS$, $g(\vect{x}) \in \evalC{\qo}{\D,\KBDM}$, which implies    	that the constants from $g(\vect{x})$ are transferred by $\map$ to the ontological level. For any variable $x \in \vect{x}$, let $c = g(x)$, and let $m \in \map$ and $h_m$ be a homomorphism  from $\body{m}$ to $D$ that map a variable $y$ from $\fr{m}$ to $c$. 
        The substitution $g$ can be seen as a bijective mapping from $\terms{\cq_\PS}$ to $\terms{D}$ (which maps variables to fresh constants and constants to themselves). Then, the composition $h = g^{-1} \circ h_m$ defines a homomorphism from $\body{m}$ to  $\cq_\PS$ such that $h(y) = x$, i.e., $x \in h(\fr{m})$. 
        
        $(\Leftarrow)$ The following is known: \emph{(i)} for any $S$-database $D$ and any homomorphism $h_s$ from $\cq_\PS$ to $D$, there is a homomorphism $h'$ from $\map(\cq_\PS)$ to $\map(D)$ such that, for all variable $x \in 
        \vars{\cq_\PS} \cap \vars{\map(\cq_\PS)}$, $h'(x) = h_s(x)$. 
        Now, let $D$ be an $S$-database and $h_s$ be a homomorphism from $\cq_\PS$ to $D$ with $\vect{c} = h_s(\vect{x})$. If for all $x \in \vect{x}$ there are $m \in \map$ and a homomorphism $h$ from $\body{m}$ to $\cq_\PS(\vect{x})$ such that $x \in h(\fr{m})$, then $\vect{x} \subseteq \vars{\map(\cq_\PS)}$, 
        hence, by \emph{(i)}, there is a homomorphism $h'$ from $\map(\cq_\PS)$ to $\map(D)$ such that $h'(\vect{x}) =  h_s(\vect{x})  = \vect{c}$, i.e., $\vect{c} \in \map(\cq_\PS)(\map(D))$. Therefore, $\vect{c} \in \evalC{\map(\cq_\PS)}{\D,\KBDM}$. We conclude that $\map(\cq_\PS)$ is a complete $\Sigma$-abstraction of $\cq_\PS$. 
        
                \item  For a UCQ$^{\neq,\textbf{C}} \ucq_\PS(\vect{x})$: 

        $(\Rightarrow)$ Let $\cq_i \in \ucq_\PS$. We have $\cq_i \sqsubseteq \ucq_\PS$, i.e., \emph{(i)} for every $S$-database $D$, $\eval{\cq_i}{\D}\subseteq\eval{\qs}{\D}$.
        Let $\qo$ be a complete $\KBDM$-abstraction of $\ucq_\PS$, i.e., such that \emph{(ii)} for every $S$-database $D$, $\eval{\qs}{\D}\subseteq\evalC{\qo}{\D,\KBDM}$. 
        From  \emph{(i)} and  \emph{(ii)}, we conclude that $\qo$ is a complete $\KBDM$-abstraction of $\cq_i$.

        $(\Leftarrow)$ For each $\cq_i \in \qs$, let $\qo^i$ be a complete $\KBDM$-abstraction of $\cq_i$. Let $\qo$ be the union of all $\qo^i$: since it is a complete $\KBDM$-abstraction of each CQ $\cq_i \in \qs$, it is also a complete $\KBDM$-abstraction of $\qs$.
        \end{itemize}
\end{proof}

\recallContent[Proposition]{prop-lemma-3-notes}

\begin{proof}
    \begin{enumerate} Assume $\Sigma(D) \models \map(D')$. 
        \item Note that $\modK{\D'}{\KBDM} \subseteq \modK{\D'}{\map}$ is always true, by the monotonicity of first-order logic. We will prove that $\sol{\KBDM}{D} \subseteq \sol{\KBDM}{D'}$. Let $J \in \sol{\KBDM}{D}$, i.e., $D \cup J \models \map$ and $J \models \erules$. 
    We have to prove that  $D' \cup J \models \map$. 
      By assumption $\map(D')$ can be mapped (by homomorphism) to $\KBDM(D)$, and, $\KBDM(D)$ being a universal model for $(D,\Sigma)$, it can be mapped to $J$. Hence, $\map(D')$ can be mapped to $J$, i.e.,  $J \models \map(D')$. Therefore, $D' \cup J \models D' \cup \map(D')$. Furthermore, by definition of the chase,  $D' \cup \map(D') \models \map$, hence  $D' \cup J \models \map$. We conclude that  $J \in \sol{\KBDM}{D'}$.    
        \item Let $\vect{c} \in \evalC{\qo}{D',\map}$.     
        By definition of certain answers, $\vect{c} \in \qo(J)$, for all $J \in \sol{\map}{D'}$. Since $\sol{\KBDM}{D} \subseteq \sol{\map}{D'}$ by Point 1, we also have that $\vect{c} \in \qo(J')$, for each $J' \in \sol{\KBDM}{D}$, hence $\vect{c} \in \evalC{\qo}{D,\KBDM}$. 
    \end{enumerate}
\end{proof}

\textbf{Lemma \ref{lemma-split}} is obtained as a corollary of the following proposition: 
\medskip

\begin{proposition}
 Let $\sigma^i_\PS$ be a grounding of $q_\PS^i \in \textit{split}(\qs)$. Then, for any mapping rule $m$:
\begin{enumerate}
\item For all $h: \body{m} \rightarrow q_\PS^i$, \\
$\sigma^i_\PS \circ h$ is a homomorphism from $\body{m}$ to $\sigma^i_\PS(q_\PS^i)$;
\item For all $h':  \body{m}  \rightarrow \sigma^i_\PS(q_\PS^i)$, \\
$(\sigma^i_\PS)^{-1} \circ h'$ is a homomorphism from $\body{m}$ to $q_\PS^i$. 
\end{enumerate}
\end{proposition}

Note that Point 1 holds for any CQ$^{\neq,\textbf{C}}$, while Point 2 relies on the injectivity of $\sigma^i_\PS$ and the special handling of constants from $\body{m}$.

\begin{proof}$~$

    \begin{enumerate}
        \item Since $h$ is a homomorphism from $\body{m}$ to $q_\PS^i$ and $\sigma^i_\PS$ is a homomorphism from $q_\PS^i$, $\sigma^i_\PS \circ h$ is a homomorphism from $\body{m}$ to $\sigma^i_\PS(q_\PS^i)$.
        
        \item Wlog, we assume that $m$ and $q_\PS^i$ do not share any variable.  Let $f_{\sigma^i_\PS}$ be the function from $\terms{q_\PS^i}$ to $\terms{\sigma^i_\PS(q_\PS^i)}$ obtained from $\sigma^i_\PS$ by extending its domain (and range) to the constants of $q_\PS^i$, such that each such constant is mapped to itself.  The key point is that $f_{\sigma^i_\PS}$ is injective. Indeed, by construction, there is a $\neq$-atom between any pair of distinct terms in $q_\PS^i$. 
        Hence, $f_{\sigma^i_\PS}$ admits an inverse relation $f_{\sigma^i_\PS}^{-1}$, which is a function (and even a bijection) from $\terms{\sigma^i_\PS(q_\PS^i)}$ to $\terms{q_\PS^i}$.

        Let $f_{h'}$ be the function from $\terms{\body{m}}$ to $\terms{\sigma^i_\PS(q_\PS^i)}$ obtained from $h'$ by extending its domain (and range) to the constants of   $q_\PS^i$, such that each such constant is mapped to itself. Let $f_h$ be the function from $\terms{\body{m}}$ to $\terms{q_\PS^i}$, defined as follows: $f_h =  f_{\sigma^i_\PS}^{-1} \circ f_{h'}$. Let $h$ be the restriction of  $f_h$ to the domain $\vars{{\body{m}}}$. We show that $h$ is a homomorphism from $\body{m}$ to $q_\PS^i$. 
        
        For any atom $a \in \body{m}$, it holds that $f_h(a) \in q_\PS^i$: indeed,  $f_{h'}(a) \in \sigma^i_\PS(q_\PS^i)$ because $f_{h'}(a) = h'(a)$ and $h'$ is a homomorphism to $\sigma^i_\PS(q_\PS^i)$, and, since $f_{\sigma^i_\PS}^{-1}$ maps each atom of $\sigma^i_\PS(q_\PS^i)$ to an atom of $q_\PS^i$, 
         we obtain
        $f_{\sigma^i_\PS}^{-1}(f_{h'}(a)) \in q_\PS^i$. So, $f_h(\body{m}) \subseteq q_\PS^i$. 
         It remains to check that $f_h$ is the identity on the constants from $\body{m}$. By construction, $f_{h'}$ is the identity on these constants. And this is also true for $f_{\sigma^i_\PS}^{-1}$ by construction of $\textit{split}(\qs)$: there is a $\neq$-atom between each variable from $q_\PS^i \in \textit{split}(\qs)$ and each constant from $\body{m}$, hence  $\sigma^i_\PS$ cannot ground a variable from  $q_\PS^i$ by a constant from $\body{m}$; so, $f_{\sigma^i_\PS}^{-1}$ is also the identity on the constants from $\body{m}$.  Hence, for any atom $a \in \body{m}$, $f_h(a) = h(a)$, i.e., $f_h(\body{m}) = h(\body{m})$. 
 \end{enumerate}\end{proof}


\recallContent[Theorem]{th-min-complete} \emph{For any mapping $\map$ and any UCQ$^{\neq,\textbf{C}}$ $\qs$ that has a complete $\map$-abstraction, there is a UCQ$^{\neq,\textbf{C}}$ $\qo$ such that, for any $\Sigma$ with mapping $\map$,  $\qo$ is a minimally complete $\Sigma$-abstraction of $\qs$. }

\begin{proof}
Let $\Sigma$ be any OBDA specification with mapping $\map$. 
Let $\qo(\tuple{x}) = \map(\textit{split}(\qs))$. 
We show that $\qo$ is a minimally complete $\Sigma$-abstraction of $\qs$. 

We first point out that $\qo$ is a complete  $\Sigma$-abstraction of $\qs$. Since $\qs$ has a complete $\Sigma$-abstraction, and $\qs$ is logically equivalent to $\textit{split}(\qs)$, $\textit{split}(\qs)$ also has a complete $\Sigma$-abstraction. 
By Prop. \ref{prop:existence_comp_trans}, $\textit{split}(\qs)$ has a complete $\Sigma$-abstraction iff each $q^{i}_S \in \textit{split}(\qs)$ has one. 
By the properties of the chase, we know that if $q^{i}_S$ has a complete $\Sigma$-abstraction, then $\map(q^{i}_S)$ is such abstraction. 
Hence, $\map(\textit{split}(\qs)) = \qo$ is a complete $\Sigma$-abstraction of $\qs$. 

To prove that $\qo$ is minimally complete, we consider any $D$ and $\tuple{c} \in  \eval{\qo}{\Sigma(\D)}$ (i.e., $\tuple{c}$ is a certain answer to $\qo$ on $(D, \Sigma)$) and we show that $\tuple{c}$ is a certain answer to \emph{any} complete $\Sigma$-abstraction of $\qs$ on $(D, \Sigma)$. 
Let $q^{i}_O \in \qo$ that maps by $h_i$ to $\Sigma(\D)$ with $h_i(\tuple{x}) = \tuple{c}$. 
Let $q^{i}_S \in \textit{split}(\qs)$ such that $q^{i}_O = \map(q^{i}_S)$. 
Consider $h_i(q^{i}_S)$: all the variables shared between $q^{i}_S$ and $q^{i}_O$ have been instantiated, in particular, $\tuple{x}$ in $q^{i}_S$ has been instantiated by $\tuple{c}$ in $h_i(q^{i}_S)$. 
Let $\sigma_i$ be a grounding of  $h_i(q^{i}_S)$, and let $D_i  = \sigma_i (h_i(q^{i}_S))$. 
Since $q^{i}_S$ maps to $D_i$ by the homomorphism defined by $\sigma_i  \circ h_i$, and $\sigma_i(h_i(\tuple{x})) = h_i(\tuple{x}) = \tuple{c}$, we have $\tuple{c} \in \eval{q^{i}_S}{D_i}$. 
\\
From Lemma \ref{lemma-split}, $\map(\sigma_i(h_i(q^{i}_S))) \equiv \sigma_i(h_i(\map(q^{i}_S))$, i.e., $\map(D_i) \equiv \sigma_i(h_i(q^{i}_O))$. Moreover,
$\sigma_i(h_i(q^{i}_O)) = h_i(q^{i}_O)$, since the domain of $\sigma_i$ does not contain any variable from $h_i(q^{i}_O)$, hence 
 $\map(D_i) \equiv h_i(q^{i}_O)$.  In particular, $\map(D_i)$ maps to $h_i(q^{i}_O)$, and since $h_i(q^{i}_O) \subseteq \Sigma(D)$,  we have $\map(D_i)$ maps to $\Sigma(D)$. 
So, by Prop. \ref{prop-lemma-3-notes}, for all ontological query $Q$, $\evalC{Q}{D_i, \Sigma} \subseteq \evalC{Q}{D, \Sigma}$. 

Let $Q$ be any complete $\Sigma$-abstraction of $\qs$. Since $\tuple{c} \in \eval{q^{i}_S}{D_i}$,  we have $\tuple{c} \in \eval{Q_S}{D_i}$, hence $\tuple{c} \in \evalC{Q}{D_i, \Sigma}$. Since $\evalC{Q}{D_i, \Sigma} \subseteq \evalC{Q}{D, \Sigma}$, we have $\tuple{c} \in \evalC{Q}{D, \Sigma}$.
 We conclude that $\qo$ is a minimally-complete $\Sigma$-abstraction of $\qs$. 

\end{proof}


\section{Complements to Section \ref{sec:max_sound}}

\medskip 
\subsection{Known results on the (non-)existence of maximally sound abstractions}

\medskip
As shown in  \cite{Cima2019SemanticCO}, a UCQ may not have a maximally sound $\map$-abstraction in that class. The following example illustrates such a case, and it is still valid if the target class is UCQ$^{\neq,\textbf{C}}$. 

\begin{example}[from \cite{Cima2023TheNO}, proof of Th. 5.2.]\label{ex:maximally_sound_ucq_does_not_always_exist}
Let $\map$ be the following (GAV) mapping: 
$$\map = 
\\
\begin{cases}
        s_1(x) &\rightarrow q(x)\\
        s_2(x) \land s_5(x) &\rightarrow t(x)\\
        s_1(y) \land s_3(x,y) &\rightarrow p(x,y)\\
        s_2(x) \land s_4(x,y) &\rightarrow p(x,y)
    \end{cases}$$ \\
    Let $\cq_S = \exists u. s_1(u) \land s_2(u)$.  
 Any maximally sound $\map$-abstraction of $\cq_S$ is equivalent to an infinite union of pairwise incomparable Boolean CQs  of the following shape: 
    $$\cq_O^n= q(u_0) \land \left(\bigwedge_{i = 1}^n p(u_{i-1}, u_{i}) \right) \land t(u_n)~\text{for} ~n \in \mathbb{N}$$

Hence, it is not possible to express a maximally sound $\map$-abstraction of $\cq_S$ as a UCQ (nor a UCQ$^{\neq,\textbf{C}}$). 

\medskip
Let us insist on the importance of comparing ontological queries with $\sqsubseteq_\map$ and not simply $\sqsubseteq$. For that, we can rely on the following equivalence:  
$\cq_O^i \sqsubseteq_\map \cq_O^j$ iff $\map^-(\cq_O^i) \sqsubseteq  \map^-(\cq_O^j)$. Note that if the atom $s_5(x)$ in the second mapping rule were removed, then the infinite union of $\cq_O^n$ ($n \in \mathbb{N}$) would become equivalent to $\cq_O^0$, because, for all $n\in \mathbb{N}$, we would have $\map^-(\cq_O^n) \sqsubseteq  \map^-(\cq_O^0) = \exists u_0. s_1(u_0) \land s_2(u_0)$; then $\cq_O^0$ would even be a perfect abstraction of  $\cq_S$. 
\end{example}
 
In \cite{Cima2019SemanticCO}, it is also shown that a maximally sound $\Sigma$-abstraction is expressible as a UCQ (and can be computed) in a quite restricted case:  $\cq_S$ is a (U)CQJFE, $\map$ is a pure GAV mapping and the ontology is DL-Lite$_{RDFS}$.  
 A GAV mapping is pure if each variable in the head of a rule occurs only once. DL-Lite$_{RDFS}$ is the restriction of DL-Lite$_{\mathcal{R}}$ to its Datalog part, i.e., rule heads have no existential variables.

\subsection{Proofs of Section \ref{sec:max_sound}}
Let  $\mapa$ be an abstract mapping. The \emph{domain} of $\mapa$ is defined as $\text{dom}(\mapa) = \{I_1 \mid (I_1,I_2) \in \mapa\}$ and its \emph{set of models w.r.t. an instance} $I_1$ is defined as $\sol{\mapa}{I_1} = \{I_2 \mid (I_1,I_2) \in \mapa\}$.

\paragraph{Definition of a maximum recovery.} On can find two equivalent definitions of a maximum recovery in the literature. The first one, given in Section \ref{sec-max-sound}, comes from \cite{Arenas2009InvertingSM}; in that paper, they actually define the notion of a $\mathbb{C}$-maximum recovery, i.e., a maximum recovery for the class of queries $\mathbb{C}$; when $\mathbb{C}$ can be any class of queries, it is a maximum recovery. This definition has the advantage of being intuitive and well suited to discussing query abstraction; however in the proofs, we also rely on an alternative definition of a maximum recovery, from \cite{Arenas2009TheRO}, which is the following. Let $\mapa$ be an abstract mapping from $\source$ to $\target$ and $\mapa'$ be an abstract mapping from $\target$ to $\source$. Then $\mapa'$ is said to be a \emph{recovery} of $\mapa$ if $(I, I) \in \mapa \bullet \mapa'$ for every instance $I \in \text{dom}(\mapa)$. Let $\mapa'$ and $\mapa''$ be recoveries of $\mapa$. Then $\mapa'$ is said to be \emph{at least as informative as} $\mapa''$ for $\mapa$, which is denoted by $\mapa'' \preceq_{\mapa} \mapa'$, if $\mapa \bullet \mapa' \subseteq \mapa \bullet \mapa''$. Finally, $\mapa'$ is a \emph{maximum recover}y of $\mapa$ if for every recovery $\mapa''$ of $\mapa$, it is the case that
$\mapa'' \preceq_{\mapa} \mapa'$. 

\medskip

Next Proposition \ref{prop:equivalence_maximum_recovery_reduced_recovery_composition} gives useful properties of maximum recoveries. This proposition uses the notion of a reduced recovery: an abstract mapping $\mapa'$ is a \emph{reduced recovery} of $\mapa$ if $\mapa'$ is a recovery of $\mapa$ and for every $(I_1, I_2) \in \mapa \bullet \mapa'$, we have $I_2 \in \text{dom}(\mapa)$.

\begin{proposition}[Proposition 3.8 from \cite{Arenas2009TheRO}]\label{prop:equivalence_maximum_recovery_reduced_recovery_composition}
    Let $\mapa$ and $\rmap$ be abstract mappings. Then the following conditions are equivalent:
    \begin{enumerate}
        \item $\rmap$ is a maximum recovery of $\mapa$.
        \item  $\rmap$ is a reduced recovery of $\mapa$ and $\mapa = \mapa \bullet \rmap \bullet \mapa$.
        \item $\rmap$ is a recovery of $\mapa$ and for every $(I_1, I_2) \in \mapa \bullet \rmap$, it is the case that $\emptyset \subsetneq \sol{\mapa}{I_2} \subseteq \sol{\mapa}{I_1}$.
    \end{enumerate}
\end{proposition}

We also use the characterisation given by next Theorem \ref{thm:link_max_recov_witness_solution},  which involves the notion of a witness model:  

\begin{definition}[Witness Model]
An instance $J$ on $\target$ is a \emph{witness model} for $I$ under $\map_\mathcal{A}$ if:
\begin{enumerate}
    \item $(I,J) \in \map_\mathcal{A}$, and 
    \item for every instance $I'$ on $\source$ such that $(I',J) \in \map_\mathcal{A}$, we have $\sol{\map_\mathcal{A}}{I} \subseteq \sol{\map_\mathcal{A}}{I'}$.
\end{enumerate}
\end{definition}

\begin{theorem}[Theorem 3.12 from \cite{Arenas2009TheRO}]\label{thm:link_max_recov_witness_solution}
A mapping $\mapa$ has a maximum recovery if and only if for
every $I \in \operatorname{dom}(\mapa)$, there exists a witness model for $I$ under $\mapa$.
\end{theorem}

The characterisations of a maximum recovery from Proposition \ref{prop:equivalence_maximum_recovery_reduced_recovery_composition} and Theorem \ref{thm:link_max_recov_witness_solution} still hold for abstract mappings that are relations between infinite instances. Indeed, one can check  that the proofs in \cite{Arenas2009TheRO} do not use any property of a \emph{finite} instance.

\recallContent[Lemma]{lem:link_max_recov_max_sound_trans}

\begin{proof}
\begin{sloppypar}    
    \textit{(1)} By definition of a recovery, \emph{(i)} $\certain{\mapa \bullet \rmap}{\q_\mathcal{S}}{I_\source} \subseteq \q_\mathcal{S}(I_\source)$, for any $\source$-instance $I_\source$. We  show that \emph{(ii)} $\certain{\mapa \bullet \rmap}{\q_\mathcal{S}}{I_\source} = \certain{\mapa}{\q_\target}{I_\source}$ for any $\source$-instance $\I_\source$. From \emph{(i)} and  \emph{(ii)}, we have $\certain{\mapa}{\q_\target}{I_\source} \subseteq \q_\source(I_\source)$ for any  $\source$-instance $I_\source$. 
\end{sloppypar}

\medskip
    Let $\I_\source$ be any $\source$-instance:
    \begingroup
\allowdisplaybreaks
    \begin{align*}
    \certain{\mapa \bullet \rmap}{\q_\mathcal{S}}{\I_\source} 
        &= \bigcap\limits_{(\I_\source,\I_\source') \in \mapa \bullet \rmap} \q_\mathcal{S}(\I_\source') \text{ (by definition of certain answers)} 
        \\
        &= \bigcap\limits_{(\I_\source,\I_\target) \in \mapa} \q_\target(\I_\target) 
                     \text{ (because }  \q_\target \text{ is a perfect rewriting of } \q_\mathcal{S} \text{ through } \rmap)  \\
        &= \certain{\mapa}{\q_\target}{\I_\source}
    \end{align*}
\endgroup

\begin{sloppypar}
    \textit{(2)} The proof is by contradiction. Assume there are a query $\q_\target'$, which is a sound $\mapa$-abstraction of $\q_\mathcal{S}$, and an $\source$-database $D$ such that $\certain{\mapa}{\q_\target}{D} \subsetneq \certain{\mapa}{\q_\target'}{D}$. This implies there is $\vect{c} \in \certain{\mapa}{\q_\target'}{D}$ such that $\vect{c} \notin \certain{\mapa}{\q_\target}{D}$. And since we know from  \textit{(1)} that $\certain{\mapa \bullet \rmap}{\q_\mathcal{S}}{D} = \certain{\mapa}{\q_\target}{D}$, we have $\vect{c} \notin \certain{\mapa \bullet \rmap}{\q_\mathcal{S}}{D}$. This implies there exists $(D,D') \in \mapa \bullet \rmap$  such that $\vect{c} \notin \q_\mathcal{S}(D')$. 
By Proposition \ref{prop:equivalence_maximum_recovery_reduced_recovery_composition}, we have $\sol{\mapa}{D'} \subseteq \sol{\mapa}{D}$. 
Thus, $\certain{\mapa}{\q_\target'}{D} \subseteq \certain{\mapa}{\q_\target'}{D'}$. But since $\q_\target'$ is a sound $\mapa$-abstraction of $\q_\mathcal{S}$, we also have $\certain{\mapa}{\q_\target'}{D'} \subseteq \q_\mathcal{S}(D')$. 
  Therefore, $\certain{\mapa}{\q_\target'}{D} \subseteq \q_\mathcal{S}(D')$, which contradicts the assumption that $\vect{c} \in \certain{\mapa}{\q_\target'}{D}$ and $\vect{c} \notin \q_\mathcal{S}(D')$.
  \end{sloppypar}  
\end{proof}

\recallContent{th-max-recovery}

\begin{proof} Follows directly from Lemma \ref{lem:link_max_recov_max_sound_trans}. 
\end{proof}

\recallContent{th:specification_max_recovery_of_obda_specification}

\begin{proof} 
First note that $\map_\lor^{\Sigma}$ is well defined because the ruleset $\erules$ is \emph{fus}, hence $\map \cup \erules$ too, which ensures that the rewriting of each rule head of $\map \cup \erules$ against $\map \cup \erules$ is a finite disjunction of CQs. 
Let $\Sigma_\mathcal{A}$ and $\mapa^\KBDM$ be the abstract mappings associated with $\Sigma$ and $\map_\lor^{\Sigma}$, respectively. 
Point 2 of the theorem follows directly from its Point 1 using Lemma \ref{lem:link_max_recov_max_sound_trans}. We will now prove Point 1.  To show that $\map_\lor^{\Sigma}$ is a maximum recovery of $\Sigma$, we will rely on the third point of Proposition \ref{prop:equivalence_maximum_recovery_reduced_recovery_composition} and show that:
\begin{enumerate}
    \item[\textbf{(1)}] $\mapa^\KBDM$ is a recovery of $\Sigma_\mathcal{A}$, and
    \item[\textbf{(2)}] for every $(I_1 , I_2) \in \Sigma_\mathcal{A} \bullet \mapa^\KBDM$, it is the case that $\emptyset \subsetneq \sol{\Sigma_\mathcal{A}}{I_2} \subseteq \sol{\Sigma_\mathcal{A}}{I_1}$.
\end{enumerate}

Note that, since instances on the source schema are databases, $I_1$ and $I_2$ are databases.
 In the following, we denote by $m = \forall \vect{x} \forall \vect{y}. B[\vect{x},\vect{y}] \rightarrow \exists \vect{z}. H[\vect{x}, \vect {z}]$ a rule in $\map$, and by
 $m' =  \forall \vect{x} \forall \vect{z}. H[\vect{x}, \vect{z}] \land \textbf{C}(\vect{x}) \rightarrow \exists \vect{z'}. H'[\vect{x}, \vect{z'}]$ a rule in $ \map_\lor^{\Sigma}$, where $H'$ is a $\Sigma$-rewriting of $H$ (hence, a disjunction).  

\medskip
\textbf{(1)} We show that $\mapa^\KBDM$ is a recovery of $\Sigma_\mathcal{A}$, that is, for all $D$ over $\PS$, $(D,D) \in \Sigma_\mathcal{A} \bullet \mapa^\KBDM$, which is equivalent to: there is an instance $I$ such that $(D,I) \in \Sigma_\mathcal{A}$ and $(I,D) \in \mapa^\KBDM$. 
Let us consider $I =  \Sigma(D)$. By the properties of the chase, $(D, \Sigma(D)) \in \Sigma_\mathcal{A}$. Let us show that $(\Sigma(D), D) \in \mapa^\KBDM$, i.e., $\Sigma(D) \cup D \models \map_\lor^{\Sigma}$. 
Given any $m': H \rightarrow H' \in \map_\lor^{\Sigma}$, we show that $\Sigma(D) \cup D \models m'$. Assume that there exists a tuple of constants $\vect{c}$ such that $\Sigma(D) \models \exists \vect{z}. H[\vect{c}, \vect{z}] \land \textbf{C}[\vect{c}]$ (frontier variables can only be mapped to constants because of the $\textbf{C}$-atoms). We need to show that $D \models \exists \vect{z'}. H'[\vect{c}, \vect{z'}]$.
Let $\cq_H(\vect{x}) = \exists \vect{z}. H[\vect{x}, \vect{z}]$ be the CQ associated with $H$ and $\q_{H'}(\vect{x}) = \exists \vect{z'}. H'[\vect{x}, \vect{z'}]$ be the UCQ associated with $H'$. Since we have $\textbf{C}(\vect{c})$ and $\Sigma(D) \models \exists \vect{z}. H[\vect{c}, \vect{z}]$, we have $\vect{c} \in \cq_H(\Sigma(D))$. 
 By the properties of the chase,  $\evalC{\cq_H}{D, \Sigma} = \cq_H(\Sigma(D))$, hence 
 $\vect{c} \in \evalC{\cq_H}{D, \Sigma}$.
By definition of $m'$, $\q_{H'}$ is a (perfect) $\Sigma$-rewriting of $\cq_H$, hence  $\q_{H'}(D) = \evalC{\cq_H}{D,\Sigma}$. So, $\vect{c} \in \q_{H'}(D)$, that is  $D \models \exists \vect{z'}. H'[\vect{c}, \vect{z'}]$. We conclude that $\Sigma(D) \cup D \models \map_\lor^{\Sigma}$. 

\textbf{(2)} We show that $\mapa^\KBDM$ is maximum.
Relying on Proposition \ref{prop:equivalence_maximum_recovery_reduced_recovery_composition}, we show that for all $(D_1, D_2) \in \Sigma_\mathcal{A} \bullet \mapa^\KBDM$, we have $\emptyset \neq \sol{\Sigma_\mathcal{A}}{D_2} \subseteq \sol{\Sigma_\mathcal{A}}{D_1}$. We know that $\sol{\Sigma_\mathcal{A}}{D_2} \neq \emptyset$ as it contains at least $\Sigma(D_2)$. 
Let  $I \in \sol{\Sigma_\mathcal{A}}{D_2}$.  
Note that $I \models \erules$ (because $\erules$ and $I$ are both only over $\PO$). So, to prove that $I \in \sol{\Sigma_\mathcal{A}}{D_1}$, we only have to show that $D_1 \cup I \models \map$.
Let $m = B \rightarrow H$ be a rule in $\map$. Assume that there exists a tuple of constants $\vect{c}$ such that $D_1 \models  \exists \vect{y}. B[\vect{c},\vect{y}]$. We show that $I \models \exists \vect{z}. H[\vect{c}, \vect{z}]$.
By definition of $\map_\lor^{\Sigma}$, we have $m' = H \rightarrow H' \in \map_\lor^{\Sigma}$ where $H'$ is a $\Sigma$-rewriting of $H$. Since $(D_1,D_2) \in \Sigma_\mathcal{A} \bullet \mapa^\KBDM$, there exists an instance $J$ such that $(D_1, J) \in \Sigma_\mathcal{A}$ and $(J, D_2) \in \mapa^\KBDM$. Since $D_1 \models  \exists \vect{y}. B[\vect{c},\vect{y}]$ and $J \in \sol{\Sigma_\mathcal{A}}{D_1}$, we have $J \models \exists \vect{z}. H[\vect{c}, \vect{z}]$. And since $(J, D_2) \in \mapa^\KBDM$, we have $J \cup D_2 \models \map_\lor^{\Sigma}$ and thus $D_2 \models \exists \vect{z'}. H'[\vect{c}, \vect{z'}]$, 
i.e., $\vect{c} \in \q_{H'}(D_2)$. 
Since $\q_{H'}$ is a (perfect) $\Sigma$-rewriting of $\cq_H$, $\q_{H'}(D_2) = \evalC{\cq_H}{D_2, \Sigma}$, hence $\vect{c} \in \evalC{\cq_H}{D_2, \Sigma}$. By definition of certain answers, $\vect{c}$ is an answer to $\cq_H$ on all instances from $\sol{\Sigma_\mathcal{A}}{D_2}$, in particular $\vect{c} \in \cq_H(I)$, i.e., $I \models \exists \vect{z}. H[\vect{c}, \vect{z}]$. So, $D_1 \cup I \models \map$. We conclude that $I \in \sol{\Sigma_\mathcal{A}}{D_1}$.
\end{proof}

\subsection{Maximum recovery of an OBDA specification (additional result)} 

From Theorem \ref{th:specification_max_recovery_of_obda_specification}, we know that every abstract mapping associated with an OBDA specification where $\erules$ is FO-rewritable has a maximum recovery that can be expressed as a disjunctive mapping, and we furthermore know how to compute this maximum recovery. The following additional theorem states that every abstract mapping associated with an OBDA specification (without any condition on $\erules$) has a maximum recovery. We leave open the question of which mapping language would allow one to specify such a maximum recovery when $\erules$ is not FO-rewritable.

\saveContent{theorem}{th:every_abstract_obda_specification_has_max_recov}{Any abstract mapping associated with an OBDA specification has a maximum recovery.}

\begin{proof}[Proof] 
    By Theorem \ref{thm:link_max_recov_witness_solution}, we know that an abstract mapping $\map_\mathcal{A}$ from $\source$ to $\target$ has a maximum recovery if and only if for every instance $I$ on $\source$ (i.e., database on $\source$), there exists a witness model for $I$ under $\mapa$. 
Let $\Sigma_\mathcal{A}$ be an abstract mapping associated with an OBDA specification $\Sigma$. Let $D$ be any database on $\PS$ and let $J = \Sigma(D)$. By the properties of the chase, $J$ is a universal model for $(D,\KBDM)$. 
We will show that $J$ is a witness model for $D$ under $\Sigma_\mathcal{A}$. 

Let  $D'$ be a database over $\PS$ such that $J \in \sol{\Sigma_\mathcal{A}}{D'}$. Let $J' = \Sigma(D')$.
Since $J'$ is a universal model for $(D',\KBDM)$, there is a homomorphism from $J'$ to $J$. Since $J' = \Sigma(D')$ and $J = \Sigma(D)$, by Proposition \ref{prop-lemma-3-notes}, we have $\sol{\Sigma}{D} \subseteq \sol{\Sigma}{D'}$, i.e., $\sol{\Sigma_\mathcal{A}}{D} \subseteq \sol{\Sigma_\mathcal{A}}{D'}$. We conclude that $J$ is a witness model for $I$ under $\Sigma_\mathcal{A}$.

\end{proof}

\subsection{Non-expressiblity of CQ-maximum recoveries as conjunctive mappings} 

Let us recall that a \emph{CQ-maximum} recovery of a (GLAV) mapping is a recovery that maximally recovers the answers to CQs (and not necessarily to UCQs)  \cite{Arenas2009InvertingSM}.
 As pointed out in Section \ref{sec-max-sound}, our Theorem \ref{th-max-recovery} contradicts Theorem 4.4 from \cite{Arenas2009InvertingSM}, which states that a CQ-maximum recovery of a GLAV mapping can always be specified by a \emph{conjunctive} mapping. That paper furthermore proposes an algorithm designed to output such CQ-maximum recovery, given any GLAV mapping $\map$. But then, CQ rewriting through a CQ-maximum recovery would always be finite (hence, expressible as a UCQ or UCQ$^{\neq,\textbf{C}}$), which contradicts the fact that a maximally sound $\map$-abstraction of a CQ is not always expressible as a UCQ. 
 
 Let us take a closer look at these results and show that Theorem 4.4 does not hold true. Relying again on the mapping $\map$ from previous Example \ref{ex:maximally_sound_ucq_does_not_always_exist}, we show that \emph{(1)} the conjunctive mapping output by the algorithm from \cite{Arenas2009InvertingSM} is not a  CQ-maximum recovery of $\map$, and \emph{(2)} in fact, $\map$ does not admit any CQ-maximum recovery expressible as a conjunctive mapping. 
 
\begin{example}\label{ex:maximally_sound_ucq_does_not_always_exist_cont}
    Let $\map$ be the mapping from  Example \ref{ex:maximally_sound_ucq_does_not_always_exist}. The maximum recovery of $\map$ can be specified by 
 $$ \map' = 
 \begin{cases}
m'_1= q(x) \land \textbf{C}(x) &\rightarrow s_1(x)\\
m'_2=  t(x) \land \textbf{C}(x) &\rightarrow s_2(x) \land s_5(x)\\
m'_3=  p(x,y) \land \textbf{C}(x) \land \textbf{C}(y) &\rightarrow (s_1(y) \land s_3(x,y)) \lor (s_2(x) \land s_4(x,y))
\end{cases}
$$

The algorithm from \cite{Arenas2009InvertingSM} outputs the conjunctive mapping $\map'' = \{m_1',m_2'\}$, supposed to be a CQ-maximum recovery of $\map$.  Now, consider the database $D = \{s_1(a), s_2(a), s_2(b), s_4(a,b), s_5(b)\}$ and the CQ $\cq_S = \exists u. s_1(u) \land s_2(u)$ from Example \ref{ex:maximally_sound_ucq_does_not_always_exist}. 
 Then: $\map''(\map(D)) = \{ s_1(a), s_2(b), s_5(b)\} \not \models \cq_S$. 
If we consider $\map'$ instead of $\map''$, we have to use the disjunctive chase to apply  $\map'$ (see e.g. \cite{Leclre2023QueryRW} for details), and we obtain $\map'(\map(D)) = I_1\lor I_2$, with $I_1 = \{s_1(a), s_2(b), s_5(b), s_1(b), s_3(a,b)\}$ and $I_2 = \{s_1(a), s_2(b), s_5(b), s_2(a), s_4(a,b)\}$. Since $I_1 \models \cq_S$ and $I_2 \models \cq_S$, we have $\map'(\map(D))\models \cq_S$.

 In other words, $\emptyset = \certain{\mapa \bullet \rmap'}{\cq_S}{D} \subsetneq \certain{\mapa \bullet \rmap}{\cq_S}{D} = \{()\}$ (with $\mapa$, $\rmap$ and $\rmap'$  being the abstract mappings associated  with $\map$, $\map'$ and $\map''$, respectively). This shows that $\map''$ is not maximum for CQs.
    
\end{example}

In fact, the mapping $\map$ allows one to show the following:

\begin{proposition} There is a G(L)AV mapping that does not admit any CQ-maximum recovery  expressible as a conjunctive mapping.
\end{proposition}

\begin{proof}
Consider again Example \ref{ex:maximally_sound_ucq_does_not_always_exist}, which provides a G(L)AV mapping $\map$ and a CQ $\cq_S$ that does not admit any maximally sound $\map$-abstraction expressible as a UCQ (or UCQ$^{\neq,\textbf{C}}$). Assume that $\map$ has a CQ-maximum recovery $\map''$ that is a conjunctive mapping. 
Then, $\map''^-(\cq_S)$ is finite (i.e., it is a UCQ$^{\neq,\textbf{C}}$). Moreover, let $\map'$ be the maximum recovery of $\map$ given in Example \ref{ex:maximally_sound_ucq_does_not_always_exist_cont}, and let $Q_O = \map'^-(\cq_S)$. 
Then:
\begin{enumerate}
\item For all database $D$, $\certain{\map \bullet \map''}{\cq_S}{D} \subseteq \cq_S(D)$ because $\map''$ is a recovery.
\item For all database $D$, $\certain{\map}{\map''^-(\cq_S)}{D} =  \certain{\map \bullet \map''}{\cq_S}{D}$ by the properties of perfect query rewriting.
\item From 1 and 2: For all database $D$, $\certain{\map}{\map''^-(\cq_S)}{D} \subseteq \cq_S(D)$, i.e., $\map''^-(\cq_S)$ is a sound abstraction of $\cq_S$.
\item From 3 and $Q_O$ being a maximally sound abstraction of  $\cq_S$: For all database $D$, 
$\certain{\map}{\map''^-(\cq_S)}{D}  \subseteq \certain{\map}{Q_O}{D}$. 
\item From 4 and $Q_O$ not being equivalent to any UCQ$^{\neq,\textbf{C}}$: There is a database $D$, such that 
$\certain{\map}{\map''^-(\cq_S)}{D}  \subsetneq \certain{\map}{Q_O}{D}$. 
\item From 5 and the properties of perfect query rewriting: For this database $D$, 
$\certain{\map.\map''}{\cq_S}{D}  \subsetneq \certain{\map.\map'}{\cq_S}{D}$. 
This contradicts the assumption that $\map''$ is a CQ-maximum recovery. 
\end{enumerate}
\end{proof}

\end{document}